\documentclass[11pt]{article}

\usepackage{fullpage,times,url,bm}

\usepackage{amsthm,amsfonts,amsmath,amssymb,epsfig,color,float,graphicx,verbatim}
\usepackage{algorithm,algorithmic}
\usepackage{bbm}
\usepackage{natbib}

\usepackage{hyperref}
\hypersetup{
	colorlinks   = true, 
	urlcolor     = blue, 
	linkcolor    = blue, 
	citecolor   = black 
}

\newtheorem{theorem}{Theorem}[section]

\newtheorem{lemma}{Lemma}[section]
\newtheorem{corollary}{Corollary}[section]

\newtheorem{remark}{Remark}[section]

\usepackage{amssymb}

\usepackage{dsfont}

\newcommand{\stam}[1]{}

\newtheorem{assumption}[theorem]{Assumption}
\newtheorem{terminology}[theorem]{Terminology}

\newcommand{\bx}{\mathbf{x}}
\newcommand{\bw}{\mathbf{w}}

\newcommand{\bv}{\mathbf{v}}
\newcommand{\bz}{\mathbf{z}}

\newcommand{\co}{{\cal O}}
\newcommand{\ca}{{\cal A}}

\newcommand{\cd}{{\cal D}}

\newcommand{\ch}{{\cal H}}

\newcommand{\cl}{{\cal L}}

\newcommand{\cu}{{\cal U}}
\newcommand{\cx}{{\cal X}}

\newcommand{\cn}{{\cal N}}

\newcommand{\bbs}{{\mathbb S}}
\newcommand{\reals}{{\mathbb R}}
\newcommand{\nat}{{\mathbb N}}

\newcommand{\zero}{{\mathbf{0}}}

\newcommand{\csp}{\mathrm{CSP}}
\newcommand{\scat}{\mathrm{SCAT}}
\newcommand{\rand}{\mathrm{rand}}
\newcommand{\sat}{\mathrm{SAT}}
\newcommand{\Hcnn}{\ch_\mathrm{sign-cnn}}
\newcommand{\diag}{\mathrm{diag}}
\newcommand{\Dvec}{\cd_\mathrm{vec}}
\newcommand{\Dmat}{\cd_\mathrm{mat}}

\DeclareMathOperator*{\E}{\mathbb{E}}

\newcommand{\tn}{{\tilde{n}}}
\newcommand{\inner}[1]{\langle #1 \rangle}
\newcommand{\norm}[1]{\left\|#1\right\|}

\title{Hardness of Learning Neural Networks with \\
Natural Weights}

\author{
    Amit Daniely\thanks{School of Computer Science and Engineering, The Hebrew University, Jerusalem, Israel and Google Research Tel-Aviv, \texttt{amit.daniely@mail.huji.ac.il }}
	\and
	Gal Vardi\thanks{Weizmann Institute of Science, Israel, \texttt{gal.vardi@weizmann.ac.il}}
}
\date{}

\begin{document}

\maketitle

\begin{abstract}
Neural networks are nowadays highly successful despite strong hardness results. The existing hardness results focus on the network architecture, and assume that the network's weights are arbitrary.
A natural approach to settle the discrepancy is to assume that the network's weights are ``well-behaved" and posses some generic properties that may allow efficient learning. This approach is supported by the intuition that the weights in real-world networks are not arbitrary, but exhibit some ''random-like" properties with respect to some ''natural" distributions.
We prove negative results in this regard, and show that for depth-$2$ networks, and many ``natural" weights distributions such as the normal and the uniform distribution, most networks are hard to learn. Namely, there is no efficient learning algorithm that is provably successful for most weights, and every input distribution. It implies that there is no generic property that holds with high probability in such random networks and allows efficient learning.
\end{abstract}

\section{Introduction}

Neural networks have revolutionized performance in multiple domains, such as computer vision and natural language processing, and have proven to be a highly effective tool for solving many challenging problems. This impressive practical success of neural networks is not well understood from the theoretical point of view. In particular, despite extensive research in recent years, it is not clear which models are learnable by neural networks algorithms.

Historically, there were many negative results for learning neural networks, and it is now known that under certain complexity assumptions, it is computationally hard to learn the class of functions computed by a neural network, even if the architecture is very simple.
Indeed, it has been shown that learning neural networks is hard already for networks of depth $2$ \citep{KlivansSh06,daniely2016complexity}.
These results hold already for {\em improper learning}, namely where the learning algorithm is allowed to return a hypothesis that does not belong to the considered hypothesis class.

In recent years, researchers have considered several ways to circumvent the discrepancy  between those hardness results and the empirical success of neural networks. Namely, to understand which models are still learnable by neural networks algorithms. This effort includes proving learnability of linear models, including polynomials and kernel spaces~\citep{andoni2014learning, xie2016diverse, daniely2016toward,  daniely2017sgd, brutzkus2017sgd, jacot2018neural, du2018gradient, oymak2018overparameterized, allen2018learning, allen2018convergence, cao2019generalization, zou2019improved,  song2019quadratic, ge2019mildly, oymak_towards_2019, arora2019fine,  cao2019generalization, ziwei2019polylogarithmic, ma2019comparative, lee2019wide, daniely2019neural}, making assumptions on the input distribution~\citep{li2017convergence,brutzkus2017globally,du2017convolutional,du2017gradient,du2018improved,goel2018learning,shamir2018distribution}, the network's weights~\citep{arora2014provable,shamir2018distribution,das2019learnability,agarwal2020deep,goel2017learning}, or both~\citep{janzamin2015beating,tian2017analytical}.

In that respect, one fantastic result that can be potentially proven, is that neural networks are efficiently learnable if we assume that the network's weights are ``well-behaved". Namely, that there are some generic properties of the network's weights that allow efficient learning. This approach is supported by the intuition that the weights in real-world networks are not arbitrary, but exhibit some ''random-like" properties with respect to some ''natural weights distributions" (e.g., where the weights are drawn from a normal distribution). We say that a property of the network's weights is a {\em natural property} with respect to such a natural weights distribution, if it holds with high probability.
Existing hardness results focus on the network architecture, and assume that the weights are arbitrary.
Thus, it is unclear whether there exists a natural property that allows efficient learning.

In this work, we investigate networks with random weights, and networks whose weights posses natural properties.
We show that under various natural weights distributions {\em most networks are hard to learn}. Namely, there is no efficient learning algorithm that is provably successful for most weights, and every input distribution.
We show that it implies that learning neural networks is hard already if their weights posses some natural property.
Our hardness results are under the common assumption that refuting a random $K$-SAT formula is hard (the {\em RSAT assumption}).
We emphasize that our results are valid for {\em any} learning algorithm, and not just common neural networks algorithms.

We consider networks of depth $2$ with a single output neuron, where the weights in the first layer are drawn from some natural distribution, and the weights in the second layer are all $1$. We consider multiple natural weights distributions, for example, where the weights vector of each hidden neuron is distributed by a multivariate normal distribution, distributed uniformly on the sphere, or that each of its components is drawn i.i.d. from a normal, uniform or Bernoulli distribution. For each weights distribution, we show that learning such networks with high probability over the choice of the weights is hard. Thus, for such weights distributions, most networks are hard.
It implies that there is no generic property that holds with high probability (e.g., with probability $0.9$) in such random networks and allows efficient learning. Hence, if generic properties that allow efficient learning exist, then they are not natural, namely, they are rare with respect to all the natural weights distributions that we consider.

We also consider random neural networks of depth $2$, where the first layer is a convolutional layer with non-overlapping patches such that its filter is drawn from some natural distribution, and the weights of the second layer are all $1$. We show that learning is hard also for such networks.
It implies that there is no generic property that holds with high probability in such random convolutional networks and allows efficient learning.

\subsection*{Related work}

{\bf Hardness of learning neural networks.}
Hardness of learning neural networks in the standard (improper and distribution free) PAC model, follows from hardness of learning intersection of halfspaces. \cite{KlivansSh06} showed that, assuming the hardness of the shortest vector problem, learning intersection of $n^\epsilon$ halfspaces for a constant $\epsilon >0$ is hard. \cite{daniely2016complexity} showed that, under the RSAT assumption, learning intersection of $\omega(\log(n))$ halfspaces is hard. These results imply hardness of learning depth-$2$ networks with $n^\epsilon$ and $\omega(\log(n))$ hidden neurons (respectively).
In the agnostic model, learning halfspaces is already hard \citep{FeldmanGoKhPo06,daniely2016half}.

{\bf Learning random neural networks.}
\cite{shamir2018distribution} considers the problem of learning neural networks, where the weights are not arbitrary, but exhibit ``nice" features such as non-degeneracy or some ``random-like" appearance. The architecture of the networks that he considers is similar to ours. 
He shows that (under the RSAT assumption) no algorithm invariant to linear transformations can efficiently learn such networks if the columns of the weights matrix of the first layer are linearly independent. It implies that linear-invariant algorithms cannot learn such networks when the weights are chosen randomly.
We note that this result holds only for linearly-invariant algorithms, which is a strong restriction. Standard gradient decent methods, for example, are not linearly invariant\footnote{\cite{shamir2018distribution} shows that gradient decent can become linearly invariant if it is preceded by a certain preconditioning step.}. Our results hold for all algorithms.

In \cite{das2019learnability}, it is shown that deep random neural networks (of depth $\omega(\log(n))$) with the sign activation function, are hard to learn in the statistical query (SQ) model. This result was recently extended by \cite{agarwal2020deep} to other activation functions, including the ReLU function. While their results hold for networks of depth $\omega(\log(n))$ and for SQ algorithms, our results hold for depth-$2$ networks and for all algorithms.

Our paper is structured as follows: In Section~\ref{sec:preliminaries} we provide notations and definitions, followed by our results in Section~\ref{sec:results}. We sketch our proof ideas in Section~\ref{sec:proof idea}, with all proofs deferred to Section~\ref{sec:proofs}.

\section{Preliminaries}
\label{sec:preliminaries}

\subsection{Random Constraints Satisfaction Problems}\label{sec:CSP}
Let $\cx_{n,K}$ be the collection of {\em (signed) $K$-tuples}, that is, sequences $x=[(\alpha_1,i_1),\ldots,(\alpha_K,i_K)]$ for $\alpha_1,\ldots,\alpha_K\in \{\pm 1\}$ and distinct $i_1,\ldots,i_K\in [n]$.
Each $x\in \cx_{n,K}$ defines a function $U_x:\{\pm 1\}^n\to\{\pm 1\}^K$ by $U_x(\psi)=(\alpha_1\psi_{i_1},\ldots,\alpha_K\psi_{i_K})$.

Let $P:\{\pm 1\}^K\to \{0,1\}$ be some predicate. A {\em
  $P$-constraint} with $n$ variables is a function $C:\{\pm
1\}^n\to\{0,1\}$ of the form $C(x)=P\circ U_x$ for some $x\in\cx_{n,K}$.
An instance to the {\em CSP problem} $\csp(P)$ is
a {\em $P$-formula}, i.e., a collection $J=\{C_1,\ldots,C_m\}$ of $P$-constraints (each is specified by a $K$-tuple). The
goal is to find an assignment $\psi\in \{\pm 1\}^n$ that maximizes
the fraction of satisfied constraints (i.e.,
constraints with $C_i(\psi)=1$). We will allow CSP problems where $P$ varies with $n$ (but is still fixed for every $n$). For example, we can look of the $\lceil\log(n)\rceil$-SAT problem.

We will consider the problem of distinguishing satisfiable from random $P$ formulas (a.k.a. the problem of refuting random $P$ formulas).
For $m:\mathbb{N}\to\mathbb{N}$, we say that a randomized algorithm $\ca$ efficiently solves the problem $\csp^{\rand}_{m(n)}(P)$, if $\ca$ is a polynomial-time algorithm such that:
\begin{itemize}
\item If $J$ is a satisfiable instance to $\csp(P)$ with $n$ variables and $m(n)$ constraints, then
\[
\Pr\left(\ca(J)=\text{``satisfiable"}\right)\ge\frac{3}{4} - o_n(1)~,
\]
where the probability is over the randomness of $\ca$.
\item If $J$ is a random\footnote{To be precise, in a random formula with $n$ variable and $m$ constraints, the $K$-tuple defining each constraint is chosen uniformly, and independently from the other constraints.} instance to $\csp(P)$ with $n$ variables and $m(n)$ constraints then 
\[
\Pr\left(\ca(J)=\text{``random"}\right)\ge \frac{3}{4} - o_n(1)~,
\]
where the probability is over the choice of $J$ and the randomness of $\ca$.
\end{itemize}

\subsection{The random $K$-SAT assumption}

Unless we face a dramatic breakthrough in complexity theory,
it seems unlikely that hardness of learning can be established on standard complexity assumptions such as $\mathbf{P}\ne\mathbf{NP}$ (see \cite{ApplebaumBaXi08,daniely2013average}). Indeed, all currently known lower bounds are based on assumptions from cryptography or average case hardness. Following \cite{daniely2016complexity} we will rely on an assumption about random $K$-SAT problems which we outline below.

Let $J=\{C_1,\ldots,C_m\}$ be a random $K$-SAT formula on $n$ variables. Precisely, each $K$-SAT constraint $C_i$ is chosen independently and uniformly from the collection of $n$-variate $K$-SAT constraints.
A simple probabilistic argument shows that for some constant $C$ (depending only on $K$), if $m\ge Cn$, then $J$ is not satisfiable w.h.p.
The problem of {\em refuting random $K$-SAT formulas} (a.k.a. the problem of distinguishing satisfiable from random $K$-SAT formulas) is the problem $\csp^{\rand}_{m(n)}(\sat_K)$, where $\sat_K$ is the predicate $z_1 \vee \ldots \vee z_K$.

\stam{
The problem of {\em refuting random $K$-SAT formulas} (a.k.a. the problem of distinguishing satisfiable from random $K$-SAT formulas) seeks efficient algorithms that provide, for most formulas, a {\em refutation}. That is, a proof that the formula is not satisfiable.

Concretely, we say that an algorithm is able to refute random $K$-SAT instances with $m=m(n)\ge Cn$ clauses if on $1-o_n(1)$ fraction of the $K$-SAT formulas with $m$ constraints, it outputs ``unsatisfiable", while for {\em every} satisfiable $K$-SAT formula with $m$ constraints, it outputs ``satisfiable"\footnote{See a precise definition in section \ref{sec:CSP}}. Since such an algorithm never errs on satisfiable formulas, an output of ``unsatisfiable" provides a proof that the formula is not satisfiable.
}

The problem of refuting random $K$-SAT formulas has been extensively studied during the last 50 years.
It is not hard to see that the problem gets easier as $m$ gets larger. The currently best known algorithms~\cite{feige2004easily,coja2004strong,coja2010efficient} can only refute random instances with $\Omega\left(n^{\lceil\frac{K}{2}\rceil}\right)$ constraints for $K\ge 4$ and $\Omega\left(n^{1.5}\right)$ constraints for $K=3$.
In light of that, \cite{Feige02} made the assumption that for $K=3$, refuting random instances with $C \cdot n$ constraints, for every constant $C$, is hard (and used that to prove hardness of approximation results). Here, we put forward the following assumption.

\begin{assumption}\label{hyp:only_sat}
Refuting random $K$-$\sat$ formulas with $n^{f(K)}$ constraints is hard for some $f(K)=\omega(1)$. Namely, for every $d>0$ there is $K$ such that the problem $\csp^{\rand}_{n^d}(\sat_K)$ is hard.
\end{assumption}
\begin{terminology}
A computational problem is {\em RSAT-hard} if its tractability refutes assumption \ref{hyp:only_sat}.
\end{terminology}

In addition to the performance of best known algorithms, there is plenty of evidence to the above assumption, in the form of hardness of approximation results, and lower bounds on various algorithms, including resolution, convex hierarchies, sum-of-squares, statistical algorithms, and more. We refer the reader to \cite{daniely2016complexity} for a more complete discussion.

\subsection{Learning hypothesis classes and random neural networks}

Let $\ch \subseteq \reals^{(\reals^n)}$ be an hypothesis class. We say that a learning algorithm $\cl$ {\em efficiently (PAC) learns $\ch$} if for every target function $f \in \ch$ and every distribution $\cd$ over $\reals^n$, given only access to examples $(\bx,f(\bx))$ where $\bx \sim \cd$, the algorithm $\cl$ runs in time polynomial in $n$ and returns with probability at least $\frac{9}{10}$ (over the internal randomness of $\cl$), a predictor $h$ such that
\[
    \E_{\bx \sim \cd}\left[\left(h(\bx)-f(\bx)\right)^2\right] \leq \frac{1}{10}~.
\]

For a real matrix $W=(\bw_1,\ldots,\bw_m)$ of size $n \times m$, let $h_W: \reals^n \rightarrow [0,1]$ be the function $h_W(\bx) = \left[\sum_{i=1}^{m}[\inner{\bw_i,\bx}]_+\right]_{[0,1]}$, where $[z]_+ = \max\{0,z\}$ is the ReLU function, and $[z]_{[0,1]} = \min\{1,\max\{0,z\}\}$ is the clipping operation on the interval $[0,1]$. This corresponds to depth-$2$ networks with $m$ hidden neurons, with no bias in the first layer, and where the outputs of the first layer are simply summed and moved through a clipping non-linearity (this operation can also be easily implemented using a second layer composed of two ReLU neurons).

Let $\Dmat$ be a distribution over real matrices of size $n \times m$. We assume that $m \leq n$. We say that a learning algorithm $\cl$ {\em efficiently learns a random neural network} with respect to $\Dmat$ ($\Dmat$-random network, for short), if it satisfies the following property. For a random matrix $W$ drawn according to $\Dmat$, and every distribution $\cd$ (that may depend on $W$) over $\reals^n$, given only access to examples $(\bx,h_W(\bx))$ where $\bx \sim \cd$, the algorithm $\cl$ runs in time polynomial in $n$ and returns with probability at least $\frac{3}{4}$ over the choice of $W$ and the internal randomness of $\cl$, a predictor $h$ such that
\[
    \E_{\bx \sim \cd}\left[\left(h(\bx)-h_W(\bx)\right)^2\right] \leq \frac{1}{10}~.
\]

\begin{remark}
Learning an hypothesis class requires that for every target function in the class and every input distribution the learning algorithm succeeds w.h.p., and learning a random neural network requires that for a random target function and every input distribution the learning algorithm succeeds w.h.p. Thus, in the former case an adversary chooses both the target function and the input distribution, and in the later the target function is chosen randomly and then the adversary chooses the input distribution. Therefore, the requirement from the algorithm for learning random neural networks is weaker then the requirement for learning neural networks in the standard PAC-learning model.
We show hardness of learning already under this weaker requirement. Then, we show that hardness of learning random neural networks, implies hardness of learning neural networks with ``natural" weights under the standard PAC-learning model.
\end{remark}

\subsection{Notations and terminology}

We denote by $\cu([-r,r])$ the uniform distribution over the interval $[-r,r]$ in $\reals$; by $\cu(\{\pm r\})$ the symmetric Bernoulli distribution, namely, $Pr(r)=Pr(-r)=\frac{1}{2}$; by $\cn(0,\sigma^2)$ the normal distribution with mean $0$ and variance $\sigma^2$, and by $\cn(\zero,\Sigma)$ the multivariate normal distribution with mean $\zero$ and covariance matrix $\Sigma$. We say that a distribution over $\reals$ is symmetric if it is continuous and its density satisfies $f(x)=f(-x)$ for every $x \in \reals$, or that it is discrete and $Pr(x)=Pr(-x)$ for every $x \in \reals$.
For a matrix $M$ we denote by $s_{\min}(M)$ and $s_{\max}(M)$ the minimal and maximal singular values of $M$.
For $\bx \in \reals^n$ we denote by $\norm{\bx}$ its $L_2$ norm. We denote the $n-1$ dimensional unit sphere by $\bbs^{n-1} = \{\bx \in \reals^n: \norm{\bx}=1\}$.
For $t \in \nat$ let $[t]=\{1,\ldots,t\}$.
We say that an algorithm is efficient if it runs in polynomial time.

\section{Results}
\label{sec:results}

We show RSAT-hardness for learning $\Dmat$-random networks, where $\Dmat$ corresponds either to a fully-connected layer, or to a convolutional layer.
It implies hardness of learning depth-$2$ neural networks whose weights satisfy some natural property.
We focus on the case of networks with $\co(\log^2(n))$ hidden neurons. We note, however, that our results can be extended to networks with $q(n)$ hidden neurons, for any $q(n)=\omega(\log(n))$.
Moreover, while we consider networks whose weights in the second layer are all $1$, our results can be easily extended to networks where the weights in the second layer are arbitrary or random positive numbers.

\subsection{Fully-connected neural networks}

We start with random fully-connected neural networks.
First, we consider a distribution $\Dmat$ over real matrices, such that the entries are drawn i.i.d. from a symmetric distribution.

We say that a random variable $z$ is {\em $b$-subgaussian} for some $b>0$ if for all $t>0$ we have
\[
Pr\left(|z|>t\right) \leq 2\exp\left(-\frac{t^2}{b^2}\right)~.
\]

\begin{theorem}
\label{thm:nn iid}
Let $z$ be a symmetric random variable with variance $\sigma^2$. Assume that the random variable $z'=\frac{z}{\sigma}$ is $b$-subgaussian for some fixed $b$.
Let $\epsilon>0$ be a small constant, let $m=\co(\log^2(n))$, and let $\Dmat$ be a distribution over $\reals^{n \times m}$, such that the entries are i.i.d. copies of $z$.
Then, learning a $\Dmat$-random neural network is RSAT-hard, already if the distribution $\cd$ is over vectors of norm at most $\frac{n^{\epsilon}}{\sigma}$ in $\reals^n$.
\end{theorem}

Since the normal distribution, the uniform distribution over an interval, and the symmetric Bernoulli distribution are subgaussian (\cite{rivasplata2012subgaussian}), we have the following corollary.

\begin{corollary}
\label{cor:nn iid}
Let $\epsilon>0$ be a small constant, let $m=\co(\log^2(n))$, and let $\Dmat$ be a distribution over $\reals^{n \times m}$, such that the entries are drawn i.i.d. from a distribution $\cd_z$.
\begin{enumerate}
\item If $\cd_z=\cn(0,\sigma^2)$, then learning a $\Dmat$-random neural network is RSAT-hard, already if the distribution $\cd$ is over vectors of norm at most $\frac{n^{\epsilon}}{\sigma}$ in $\reals^n$.
\item If $\cd_z=\cu([-r,r])$ or $\cd_z=\cu(\{\pm r\})$, then learning a $\Dmat$-random neural network is RSAT-hard, even if the distribution $\cd$ is over vectors of norm at most $\frac{n^{\epsilon}}{r}$ in $\reals^n$.
\end{enumerate}
\end{corollary}

In the following theorem, we consider the case where $\Dmat$ is such that each column is drawn i.i.d. from a multivariate normal distribution.

\begin{theorem}
\label{thm:nn normal}
Let $\Sigma$ be a positive definite matrix of size $n \times n$, and let $\lambda_{\min}$ be its minimal eigenvalue.
Let $\epsilon>0$ be a small constant, let $m=\co(\log^2(n))$, and let $\Dmat$ be a distribution over $\reals^{n \times m}$, such that each column is drawn i.i.d. from $\cn(\zero,\Sigma)$.
Then, learning a $\Dmat$-random neural network is RSAT-hard, already if the distribution $\cd$ is over vectors of norm at most $\frac{n^{\epsilon}}{\sqrt{\lambda_{\min}}}$ in $\reals^n$.
\end{theorem}

We also study the case where the distribution $\Dmat$ is such that each column is drawn i.i.d. from the uniform distribution on the sphere of radius $r$ in $\reals^n$.

\begin{theorem}
\label{thm:nn sphere}
Let $m=\co(\log^2(n))$ and let $\Dmat$ be a distribution over $\reals^{n \times m}$, such that each column is drawn i.i.d. from the uniform distribution over $r \cdot \bbs^{n-1}$.
Then, learning a $\Dmat$-random neural network is RSAT-hard, already if the distribution $\cd$ is over vectors of norm at most $\co\left(\frac{n \sqrt{n} \log^4(n)}{r}\right)$ in $\reals^n$.
\end{theorem}

\stam{
\begin{remark}
Our bounds with respect to the support of $\cd$ can be improved by logarithmic factors.
Since we did not find these improvements significant, we preferred simplicity over tightness.
\end{remark}
}

From the above theorems we have the following corollary, which shows that learning neural networks (in the standard PAC-learning model) is hard already if the weights satisfy some natural property.

\begin{corollary}
\label{cor:nn property}
Let $m=\co(\log^2(n))$, and let $\Dmat$ be a distribution over $\reals^{n \times m}$ from Theorems~\ref{thm:nn normal}, \ref{thm:nn sphere}, or from Corollary~\ref{cor:nn iid}.
Let $P$ be a property that holds with probability at least $\frac{9}{10}$ for a matrix $W$ drawn from $\Dmat$. Let $\ch = \{h_W: W \in \reals^{n \times m}, \; W \text{ satisfies } P\}$ be an hypothesis class.
Then, learning $\ch$ is RSAT-hard, already if the distribution $\cd$ is over vectors of norm bounded by the appropriate expression from Theorems~\ref{thm:nn normal}, \ref{thm:nn sphere}, or Corollary~\ref{cor:nn iid}.
\end{corollary}

The corollary follows easily from the following argument:
Assume that $\cl$ learns $\ch$. If a matrix $W$ satisfies $P$ then for every distribution $\cd$, given access to examples $(\bx,h_W(\bx))$ where $\bx \sim \cd$, the algorithm $\cl$ returns with probability at least $\frac{9}{10}$ a predictor $h$ such that
\[
    \E_{\bx \sim \cd}\left[\left(h(\bx)-h_W(\bx)\right)^2\right] \leq \frac{1}{10}~.
\]
Now, let $W \sim \Dmat$. Since $W$ satisfies $P$ with probability at least $\frac{9}{10}$, then for every distribution $\cd$, given access to examples $(\bx,h_W(\bx))$ where $\bx \sim \cd$, the algorithm $\cl$ returns with probability at least $\frac{9}{10} \cdot \frac{9}{10} \geq \frac{3}{4}$ a predictor that satisfies the above inequality. Hence $\cl$ learns a $\Dmat$-random neural network.
Note that this argument holds also if $\cd$ is over vectors of a bounded norm.

\stam{
\begin{proof}
Assume that $\cl$ learns $\ch$. If a matrix $W$ satisfies $P$ then for every distribution $\cd$, given access to examples $(\bx,h_W(\bx))$ where $\bx \sim \cd$, the algorithm $\cl$ returns with probability at least $\frac{9}{10}$ a predictor $h$ such that
\[
    \E_{\bx \sim \cd}\left[\left(h(\bx)-h_W(\bx)\right)^2\right] \leq \frac{1}{10}~.
\]
Now, let $W \sim \Dmat$. Since $W$ satisfies $P$ with probability at least $\frac{9}{10}$, then for every distribution $\cd$, given access to examples $(\bx,h_W(\bx))$ where $\bx \sim \cd$, the algorithm $\cl$ returns with probability at least $\frac{9}{10} \cdot \frac{9}{10} \geq \frac{3}{4}$ a predictor that satisfies the above inequality. Hence $\cl$ learns a $\Dmat$-random neural network.
Note that this argument holds also if $\cd$ is over vectors of a bounded norm.
\end{proof}
}

\subsection{Convolutional neural networks}

We now turn to random Convolutional Neural Networks (CNN). Here, the distribution $\Dmat$ corresponds to a random convolutional layer. Our convolutional layer has a very simple structure with non-overlapping patches. Let $t$ be an integer that divides $n$, and let $\bw \in \reals^{t}$. We denote by $h_\bw^n: \reals^n \rightarrow [0,1]$ the CNN
\[
h_\bw^n(\bx)=\left[\sum_{i=1}^{\frac{n}{t}}[\inner{\bw,(x_{t(i-1)+1},\ldots,x_{t \cdot i})}]_+\right]_{[0,1]}~.
\]
Note that the $h_\bw^n$ has $\frac{n}{t}$ hidden neurons.
A random CNN corresponds to a random vector $\bw \in \reals^t$. Let $\Dvec$ be a distribution over $\reals^t$. A $\Dvec$-random CNN with $m$ hidden neurons is the CNN $h_\bw^{mt}$ where $\bw$ is drawn from $\Dvec$. Note that every such a distribution over CNNs, can be expressed by an appropriate distribution $\Dmat$ over matrices. Our results hold also if we replace the second layer of $h_\bw^n$ with a max-pooling layer (instead of summing and clipping).

We start with random CNNs, where each component in the weight vector is drawn i.i.d. from a symmetric distribution $\cd_z$ over $\reals$. In the following theorem, the function $f(t)$ bounds the concentration of $\cd_z$, and is needed in order to bound the support of the distribution $\cd$.

\begin{theorem}
\label{thm:cnn iid}
Let $n = (n'+1)\log^2(n')$.
Let $\cd_z^{n'+1}$ be a distribution over $\reals^{n'+1}$ such that each component is drawn i.i.d. from a symmetric distribution $\cd_z$ over $\reals$. Let $f(t)>0$ be a function such that $Pr_{z \sim \cd_z}(|z|<f(t))=o_t(\frac{1}{t})$.
Then, learning a $\cd_z^{n'+1}$-random CNN $h_\bw^{n}$ with $\log^2(n')=\co(\log^2(n))$ hidden neurons is RSAT-hard, already if the distribution $\cd$ is over vectors of norm at most $\frac{\log^2(n')}{f(n')}$ in $\reals^{n}$.
\end{theorem}

If $\cd_z$ is the uniform distribution $\cu([-r,r])$ then we can choose $f(t)=\frac{r}{t\log(t)}$, if it is the normal distribution $\cn(0,\sigma^2)$ then we can choose $f(t)=\frac{\sigma}{t\log(t)}$, and if it is the symmetric Bernoulli distribution $\cu(\{\pm r\})$ then we can choose $f(t)=r$. Thus, we obtain the following corollary.

\begin{corollary}
\label{cor:cnn iid}
Let $\Dvec$ be a distribution such that each component is drawn i.i.d. from a distribution $\cd_z$ over $\reals$.
\begin{enumerate}
\item If $\cd_z=\cu([-r,r])$, then learning a $\Dvec$-random CNN $h_\bw^n$ with $\co(\log^2(n))$ hidden neurons is RSAT-hard, already if $\cd$ is over vectors of norm at most $\frac{n\log(n)}{r}$.
\item If $\cd_z=\cn(0,\sigma^2)$, then learning a $\Dvec$-random CNN $h_\bw^n$ with $\co(\log^2(n))$ hidden neurons is RSAT-hard, already if $\cd$ is over vectors of norm at most $\frac{n\log(n)}{\sigma}$.
\item If $\cd_z=\cu(\{\pm r\})$, then learning a $\Dvec$-random CNN $h_\bw^n$ with $\co(\log^2(n))$ hidden neurons is RSAT-hard, already if $\cd$ is over vectors of norm at most $\frac{\log^2(n)}{r}$.
\end{enumerate}
\end{corollary}

We also consider the case where $h_\bw^n$ is a CNN such that $\bw$ is drawn from a multivariate normal distribution $\cn(\zero,\Sigma)$.

\begin{theorem}
\label{thm:cnn normal}
Let $\Sigma$ be a positive definite matrix of size $t \times t$, and let $\lambda_{\min}$ be its minimal eigenvalue. Let $n=\co(t\log^2(t))$.
Then, learning a $\cn(\zero,\Sigma)$-random CNN $h_\bw^n$ (with $\co(\log^2(n))$ hidden neurons) is RSAT-hard, already if the distribution $\cd$ is over vectors of norm at most $\frac{n\log(n)}{\sqrt{\lambda_{\min}}}$ in $\reals^{n}$.
\end{theorem}

Finally, we study the case where $h_\bw^n$ is a CNN such that $\bw$ is drawn from the uniform distribution over the sphere.

\begin{theorem}
\label{thm:cnn sphere}
Let $\Dvec$ be the uniform distribution over the sphere of radius $r$ in $\reals^t$. Let $n = \co(t\log^2(t))$.
Then, learning a $\Dvec$-random CNN $h_\bw^n$ (with $\co(\log^2(n))$ hidden neurons) is RSAT-hard, already if the distribution $\cd$ is over vectors of norm at most $\frac{\sqrt{n}\log(n)}{r}$ in $\reals^n$.
\end{theorem}

Now, the following corollary follows easily (from the same argument as in Corollary~\ref{cor:nn property}), and shows that learning CNNs (in the standard PAC-learning model) is hard already if the weights satisfy some natural property.

\begin{corollary}
\label{cor:cnn property}
Let $\Dvec$ be a distribution over $\reals^t$ from Theorems~\ref{thm:cnn normal}, \ref{thm:cnn sphere}, or from Corollary~\ref{cor:cnn iid}. Let $n = \co(t\log^2(t))$.
Let $P$ be a property that holds with probability at least $\frac{9}{10}$ for a vector $\bw$ drawn from $\Dvec$. Let $\ch = \{h_\bw^n: \bw \in \reals^t, \; \bw \text{ satisfies } P\}$ be an hypothesis class.
Then, learning $\ch$ is RSAT-hard, already if the distribution $\cd$ is over vectors of norm bounded by the appropriate expression from Theorems~\ref{thm:cnn normal}, \ref{thm:cnn sphere}, or Corollary~\ref{cor:cnn iid}.
\end{corollary}

\subsubsection{Improving the bounds on the support of $\cd$}

By increasing the number of hidden neurons from $\co(\log^2(n))$ to $\co(n)$ we can improve the bounds on the support of $\cd$.
Note that our results so far on learning random CNNs, are for CNNs with input dimension $n=\co(t \log^2(t))$ where $t$ is the size of the patches.
We now consider CNNs with input dimension $\tilde{n}=t^c$ for some integer $c>1$. Note that such CNNs have $t^{c-1}=\co(\tilde{n})$ hidden neurons.

Assume that there is an efficient algorithms $\cl'$ for learning $\Dvec$-random CNNs with input dimension $\tilde{n}=t^c$, where $\Dvec$ is a distribution over $\reals^t$. Assume that $\cl'$ uses samples with at most $\tilde{n}^d=t^{cd}$ inputs. We show an algorithm $\cl$ for learning a $\Dvec$-random CNN $h_\bw^n$ with $n=\co(t \log^2(t))$. Let $S=\{(\bx_1,h_\bw^n(\bx_1)),\ldots,(\bx_{n^{cd}},h_\bw^n(\bx_{n^{cd}}))\}$ be a sample, and let $S'=\{(\bx'_1,h_\bw^n(\bx_1)),\ldots,(\bx'_{n^{cd}},h_\bw^n(\bx_{n^{cd}}))\}$ where for every vector $\bx \in \reals^n$, the vector $\bx' \in \reals^{\tilde{n}}$ is obtained from $\bx$ by padding it with zeros. Thus, $\bx' = (\bx,0,\ldots,0)$. Note that $n^{cd}>\tilde{n}^d$. Also, note that for every $i$ we have $h_\bw^n(\bx_i)=h_\bw^{\tilde{n}}(\bx'_i)$. Hence, $S'$ is realizable by the CNN $h_\bw^{\tilde{n}}$. Now, given $S$, the algorithm $\cl$ runs $\cl'$ on $S'$ and returns an hypothesis $h(\bx)=\cl'(S')(\bx')$.

Therefore, if learning $\Dvec$-random CNNs with input dimension $n=\co(t \log^2(t))$ is hard already if the distribution $\cd$ is over vectors of norm at most $g(n)$, then learning $\Dvec$-random CNNs with input dimension $\tilde{n}=t^c$ is hard already if the distribution $\cd$ is over vectors of norm at most $g(n) < g(t^2) = g(\tilde{n}^{\frac{2}{{c}}})$. Hence we have the following corollaries.

\begin{corollary}
\label{cor:cnn iid small}
Let $\Dvec$ be a distribution over $\reals^t$ such that each component is drawn i.i.d. from a distribution $\cd_z$ over $\reals$. Let $n=t^c$ for some integer $c>1$, and let $\epsilon = \frac{3}{c}$.
\begin{enumerate}
\item If $\cd_z=\cu([-r,r])$, then learning a $\Dvec$-random CNN $h_\bw^n$ (with $\co(n)$ hidden neurons) is RSAT-hard, already if $\cd$ is over vectors of norm at most $\frac{n^{\epsilon}}{r}$.
\item If $\cd_z=\cn(0,\sigma^2)$, then learning a $\Dvec$-random CNN $h_\bw^n$ (with $\co(n)$ hidden neurons) is RSAT-hard, already if $\cd$ is over vectors of norm at most $\frac{n^{\epsilon}}{\sigma}$.
\end{enumerate}
\end{corollary}

\begin{corollary}
\label{cor:cnn normal small}
Let $\Sigma$ be a positive definite matrix of size $t \times t$, and let $\lambda_{\min}$ be its minimal eigenvalue. Let $n=t^c$ for some integer $c>1$, and let $\epsilon = \frac{3}{c}$.
Then, learning a $\cn(\zero,\Sigma)$-random CNN $h_\bw^n$ (with $\co(n)$ hidden neurons) is RSAT-hard, already if the distribution $\cd$ is over vectors of norm at most $\frac{n^\epsilon}{\sqrt{\lambda_{\min}}}$.
\end{corollary}

\begin{corollary}
\label{cor:cnn sphere small}
Let $\Dvec$ be the uniform distribution over the sphere of radius $r$ in $\reals^t$. Let $n=t^c$ for some integer $c>1$, and let $\epsilon = \frac{2}{c}$.
Then, learning a $\Dvec$-random CNN $h_\bw^n$ (with $\co(n)$ hidden neurons) is RSAT-hard, already if the distribution $\cd$ is over vectors of norm at most $\frac{n^\epsilon}{r}$.
\end{corollary}

As an example, consider a CNN $h_\bw^n$ with $n=t^c$. Note that since the patch size is $t$, then each hidden neuron has $t$ input neurons feeding into it. Consider a distribution $\Dvec$ over $\reals^t$ such that each component is drawn i.i.d. by a normal distribution with $\sigma=\frac{1}{\sqrt{t}}$. This distribution corresponds to the standard Xavier initialization. Then, by Corollary~\ref{cor:cnn iid small}, learning a $\Dvec$-random CNN $h_\bw^n$ is RSAT-hard, already if $\cd$ is over vectors of norm at most $n^{\frac{3}{c}}\sqrt{t} = n^{\frac{3}{c}} \cdot n^{\frac{1}{2c}}$. By choosing an appropriate $c$, we have that learning a $\Dvec$-random CNN $h_\bw^n$ is RSAT-hard, already if $\cd$ is over vectors of norm at most $\sqrt{n}$.

Finally, note that Corollary~\ref{cor:cnn property} holds also for the values of $n$ and the bounds on the support of $\cd$ from Corollaries~\ref{cor:cnn iid small}, \ref{cor:cnn normal small} and~\ref{cor:cnn sphere small}.

\paragraph{Open Questions}
Obvious open questions arising from our work are to obtain sharper norm bounds on the input distribution, and to avoid the non-linearity in the second layer (the clipping operation). Another direction is to extend the results to more weights distributions and network architectures. A potential positive result, is to find some useful ``unnatural" properties of the network's weights that allow efficient learning.

\section{Main proof ideas}
\label{sec:proof idea}

Let $\ch = \{h_W: W \in \reals^{n \times k}\}$ where $k=\co(\log^2(n))$. By \cite{daniely2016complexity}, learning $\ch$ is hard. We want to reduce the problem of learning $\ch$ to learning a $\Dmat$-random neural network for some fixed $\Dmat$.
Let $W \in \reals^{n \times k}$ and let $S=\{(\bx_1,h_W(\bx_1)),\ldots,(\bx_m,h_W(\bx_m))\}$ be a sample.
Let $\Dmat^M$ be a distribution over the group of $n \times n$ invertible matrices, and let $M \sim \Dmat^M$. Consider the sample $S'=\{(\bx'_1,h_W(\bx_1)),\ldots,(\bx'_m,h_W(\bx_m))\}$ where for every $i \in [m]$ we have $\bx'_i = (M^\top)^{-1} \bx_i$. Since $W^\top \bx_i = W^\top M^\top (M^\top)^{-1} \bx_i = (M W)^\top \bx'_i$, we have $h_W(\bx_i) = h_{MW}(\bx'_i)$. Thus, $S'=\{(\bx'_1,h_{MW}(\bx'_1)),\ldots,(\bx'_m,h_{MW}(\bx'_m))\}$. Note that $MW$ is a random matrix.
Now, assume that there is an algorithm $\cl'$ that learns successfully from $S'$. Consider the follow algorithm $\cl$. Given a sample $S$, the algorithm $\cl$ runs $\cl'$ on $S'$, and returns the hypothesis $h(\bx) = \cl'(S')((M^\top)^{-1} \bx)$. It is not hard to show that $\cl$ learns successfully from $S$.
Since our goal is to reduce the problem of learning $\ch$ to learning a $\Dmat$-random network where $\Dmat$ is a fixed distribution, we need $MW$ to be a $\Dmat$-random matrix. However, the distribution of $MW$ depends on both $\Dmat^M$ and $W$ (which is an unknown matrix).

Hence, the challenge is to find a reduction that translates a sample that is realizable by $h_W$ to a sample that is realizable by a $\Dmat$-random network, without knowing $W$.
In order to obtain such a reduction, we proceed in two steps. First, we show that learning neural networks of the form $h_W$ where $W \in \reals^{n \times k}$, is hard already if we restrict $W$ to a set of matrices with a special structure. Then, we show a distribution $\Dmat^M$ such that if $M \sim \Dmat^M$ and $W$ has the special structure, then $MW \sim \Dmat$.
This property, as we showed, enables us to reduce the problem of learning such special-structure networks to the problem of learning $\Dmat$-random networks.

In order to obtain a special structure for $W$, consider the class $\Hcnn^{n,k}=\{h_\bw^n: \bw \in \{\pm 1\}^{\frac{n}{k}}\}$. Note that the CNNs in $\Hcnn^{n,k}$ have $k=\co(\log^2(n))$ hidden neurons.
The networks in $\Hcnn^{n,k}$ have three important properties: (1) They are CNNs; (2) Their patches are non-overlapping; (3) The components in the filter $\bw$ are in $\{\pm 1\}$.
Hardness of learning $\Hcnn^{n,k}$ can be shown by a reduction from the RSAT problem. We defer the details of this reduction to the next sections.
Let $W=(\bw^1,\ldots,\bw^k)$ be the matrix of size $n \times k$ that corresponds to $h_\bw^n$, namely $h_W=h_\bw^n$. Note that for every $i \in [k]$ we have $\left(\bw^i_{\frac{n(i-1)}{k}+1},\ldots,\bw^i_{\frac{ni}{k}}\right)=\bw$, and $\bw^i_j=0$ for every other $j \in [n]$.

We now show a distribution $\Dmat^M$ such that if $M \sim \Dmat^M$ and $W$ has a structure that corresponds to $\Hcnn^{n,k}$, then $MW \sim \Dmat$. We start with the case where $\Dmat$ is a distribution over matrices such that each column is drawn i.i.d. from the uniform distribution on the sphere.
We say that a matrix $M$ of size $n \times n$ is a {\em diagonal-blocks matrix} if
\[
M = \begin{bmatrix}
    B^{11} & \dots & B^{1k} \\
    \vdots & \ddots & \vdots \\
    B^{k1} & \dots & B^{kk}
    \end{bmatrix}
\]
where each block $B^{ij}$ is a diagonal matrix $\diag(z_1^{ij},\ldots,z_{\frac{n}{k}}^{ij})$.
We denote $\bz^{ij}=(z_1^{ij},\ldots,z_{\frac{n}{k}}^{ij})$, and $\bz^j = (\bz^{1j},\ldots,\bz^{k j}) \in \reals^n$. Note that for every $j \in [k]$, the vector $\bz^j$ contains all the entries on the diagonals of blocks in the $j$-th column of blocks in $M$.
Let $\Dmat^M$ be a distribution over diagonal-blocks matrices, such that the vectors $\bz^j$ are drawn i.i.d. according to the uniform distribution on $\bbs^{n-1}$.
Let $W$ be a matrix that corresponds to $h_\bw^n \in \Hcnn^{n,k}$. Note that the columns of $W'=MW$ are i.i.d. copies from the uniform distribution on $\bbs^{n-1}$.
Indeed, denote $M^\top=(\bv^1,\ldots,\bv^{n})$. Then, for every line index $i \in [n]$ we denote $i=(b-1)\left(\frac{n}{k}\right)+r$, where $b,r$ are integers and $1 \leq r \leq \frac{n}{k}$. Thus, $b$ is the line index of the block in $M$ that correspond to the $i$-th line in $M$, and $r$ is the line index within the block. Now, note that
\begin{eqnarray*}
W'_{ij}
&=& \inner{\bv^i,\bw^j}
= \inner{\left(\bv^i_{(j-1)\left(\frac{n}{k}\right)+1},\ldots,\bv^i_{j\left(\frac{n}{k}\right)}\right),\bw}
= \inner{(B^{b j}_{r 1},\ldots,B^{b j}_{r \left(\frac{n}{k}\right)}),\bw}
\\
&=& B^{b j}_{r r} \cdot \bw_{r} = z^{b j}_{r} \cdot \bw_{r}~.
\end{eqnarray*}
Since $\bw_{r} \in \{\pm 1\}$, and since the uniform distribution on a sphere does not change by multiplying a subset of component by $-1$, then the $j$-th column of $W'$ has the same distribution as $\bz^j$, namely, the uniform distribution over $\bbs^{n-1}$. Also, the columns of $W'$ are independent. Thus, $W' \sim \Dmat$.

The case where $\Dmat$ is a distribution over matrices such that the entries are drawn i.i.d. from a symmetric distribution (such as $\cu([-r,r])$, $\cn(0,\sigma^2)$ or $\cu(\{\pm r\})$) can be shown in a similar way. The result for the case where $\Dmat$ is such that each columns is drawn i.i.d. from a multivariate normal distribution $\cn(\zero,\Sigma)$ cannot be obtained in the same way, since a $\cn(\zero,\Sigma)$ might be sensitive to multiplication of its component by $-1$. However, recall that a vector in $\reals^n$ whose components are i.i.d. copies from $\cn(0,1)$ has a multivariate normal distribution $\cn(\zero,I_n)$. Now, since every multivariate normal distribution $\cn(\zero,\Sigma)$ can be obtained from $\cn(\zero,I_n)$ by a linear transformation, we are able to show hardness also for the case where $\Dmat$ is such that each column is drawn i.i.d. from $\cn(\zero,\Sigma)$.

Recall that in our reduction we translate $S=\{(\bx_1,h_W(\bx_1)),\ldots,(\bx_m,h_W(\bx_m))\}$ to $S'=\{(\bx'_1,h_W(\bx_1)),\ldots,(\bx'_m,h_W(\bx_m))\}$ where for every $i \in [m]$ we have $\bx'_i = (M^\top)^{-1} \bx_i$. Therefore, we need to show that our choice of $M$ is such that it is invertible with high probability. Also, since we want to show hardness already if the input distribution $\cd$ is supported on a bounded domain, then we need to bound the norm of $\bx'_i$, with high probability over the choice of $M$. This task requires a careful analysis of the spectral norm of $(M^\top)^{-1}$, namely, of $\left(s_{\min}(M)\right)^{-1}$.

The proofs of the results regarding random CNNs follow the same ideas. The main difference is that in this case, instead of multiplying $\bx_i$ by $(M^\top)^{-1}$, we multiply each patch in $\bx_i$ by an appropriate matrix.

\subsection*{Proof structure}

We start with a few definitions.
We say that a sample $S = \{(\bx_i,y_i)\}_{i=1}^m \in (\reals^n \times \{0,1\})^m$ is {\em scattered} if $y_1,\ldots,y_m$ are independent fair coins (in particular, they are independent from $\bx_1,\ldots,\bx_m$).
We say that $S$ is {\em contained in $A \subseteq \reals^n$} if $\bx_i \in A$ for every $i \in [m]$.

Let $\ca$ be an algorithm whose input is a sample $S = \{(\bx_i,y_i)\}_{i=1}^{m(n)} \in (\reals^n \times \{0,1\})^{m(n)}$ and whose output is either ``scattered" or ``$\ch$-realizable",  where $\ch$ is an hypothesis class. We say that $\ca$ distinguishes size-$m$ $\ch$-realizable samples from scattered samples if the following holds.
\begin{itemize}
\item If the sample $S$ is scattered, then
\[
\Pr\left(\ca(S)=\text{``scattered"}\right) \ge \frac{3}{4}-o_n(1)~,
\]
where the probability is over the choice of $S$ and the randomness of $\ca$.

\item If the sample $S$ satisfies $h(\bx_i)=y_i$ for every $i \in [m]$ for some $h \in \ch$, then
\[
\Pr\left(\ca(S)=\text{``$\ch$-realizable"}\right)\ge\frac{3}{4}-o_n(1)~,
\]
where the probability is over the randomness of $\ca$.
\end{itemize}
We denote by $\scat_{m(n)}^A(\ch)$ the problem of distinguishing size-$m(n)$ $\ch$-realizable samples that are contained in $A \subseteq \reals^n$ from scattered samples.

Now, let $\ca'$ be an algorithm whose input is a sample $S = \{(\bx_i,y_i)\}_{i=1}^{m(n)} \in (\reals^n \times \{0,1\})^{m(n)}$ and whose output is either ``scattered" or ``$\Dmat$-realizable". We say that $\ca'$ distinguishes size-$m$ $\Dmat$-realizable samples from scattered samples if the following holds.
\begin{itemize}
\item If the sample $S$ is scattered, then
\[
\Pr\left(\ca'(S)=\text{``scattered"}\right) \ge \frac{3}{4}-o_n(1)~,
\]
where the probability is over the choice of $S$ and the randomness of $\ca'$.

\item If the sample $S$ satisfies $h_W(\bx_i)=y_i$ for every $i \in [m]$, where $W$ is a random matrix drawn from $\Dmat$, then
\[
\Pr\left(\ca'(S)=\text{``$\Dmat$-realizable"}\right)\ge\frac{3}{4}-o_n(1)~,
\]
where the probability is over the choice of $W$ and the randomness of $\ca'$.
\end{itemize}
We denote by $\scat_{m(n)}^A(\Dmat)$ the problem of distinguishing size-$m(n)$ $\Dmat$-realizable samples that are contained in $A \subseteq \reals^n$ from scattered samples.
In the case of random CNNs, we denote by $\scat_{m(n)}^A(\Dvec,n)$ the problem of distinguishing size-$m(n)$ scattered samples that are contained in $A$, from samples that are realizable by a random CNN $h_\bw^n$ where $\bw \sim \Dvec$.

Recall that $\Hcnn^{n,m}=\{h_\bw^n: \bw \in \{\pm 1\}^{\frac{n}{m}}\}$.
As we described, hardness of learning $\Dmat$-random neural networks where the distribution $\cd$ is supported on a bounded domain, can be shown by first showing hardness of learning $\Hcnn^{n,m}$ with some $m=\co(\log^2(n))$, where the distribution $\cd$ is supported on some $A' \subseteq \reals^n$, and then reducing this problem to learning $\Dmat$-random networks where the distribution $\cd$ is supported on some $A \subseteq \reals^n$.
We can show RSAT-hardness of learning $\Hcnn^{n,m}$ by using the methodology of \cite{daniely2016complexity} as follows:
First, show that if there is an efficient algorithm that learns $\Hcnn^{n,m}$ where the distribution $\cd$ is supported on $A' \subseteq \reals^n$, then there is a fixed $d$ and an efficient algorithm that solves $\scat_{n^d}^{A'}(\Hcnn^{n,m})$, and then show that for every fixed $d$, the problem $\scat_{n^d}^{A'}(\Hcnn^{n,m})$ is RSAT-hard.

Our proof follows a slightly different path than the one described above.
First, we show that if there is an efficient algorithm that learns $\Dmat$-random neural networks where the distribution $\cd$ is supported on $A \subseteq \reals^n$, then there is a fixed $d$ and an efficient algorithm that solves $\scat_{n^d}^A(\Dmat)$.
Then, we show that for every fixed $d$, the problem $\scat_{n^d}^{A'}(\Hcnn^{n,m})$ with some 
$A' \subseteq \reals^n$, is RSAT-hard.
Finally, for the required matrix distributions $\Dmat$ and sets $A$, we show a reduction from $\scat_{n^d}^{A'}(\Hcnn^{n,m})$ to $\scat_{n^d}^A(\Dmat)$.
The main difference between this proof structure and the one described in the previous paragraph, is that here, for every distribution $\Dmat$ we need to show a reduction from $\scat_{n^d}^{A'}(\Hcnn^{n,m})$ to $\scat_{n^d}^A(\Dmat)$ (which are decision problems), and in the previous proof structure for every $\Dmat$ we need to show a reduction between learning $\Hcnn^{n,m}$ and learning $\Dmat$-random networks.
We chose this proof structure since here the proof for each $\Dmat$ is a reduction between decision problems, and thus we found the proofs to be simpler and cleaner this way.
Other than this technical difference, both proof structures are essentially similar and follow the same ideas.

The case of random CNNs is similar, except that here we show, for each distribution $\Dvec$ over vectors, a reduction from $\scat_{n^d}^{A'}(\Hcnn^{n,m})$ to $\scat_{n^d}^A(\Dvec,n)$.

\section{Proofs}
\label{sec:proofs}


\subsection{Learning $\Dmat$-random networks is harder than $\scat_{n^d}^A(\Dmat)$}


\begin{theorem}
\label{thm:learn vs scat}
Let $\Dmat$ be a distribution over matrices. 
Assume that there is an algorithm that learns $\Dmat$-random neural networks,
where the distribution $\cd$ is supported on $A \subseteq \reals^n$.
Then, there is a fixed $d$ and an efficient algorithm that solves $\scat_{n^d}^A(\Dmat)$.
\end{theorem}
\begin{proof}
Let $\cl$ be an efficient learning algorithm that learns $\Dmat$-random neural networks where the distribution $\cd$ is supported on $A$. Let $m(n)$ be such that $\cl$ uses a sample of size at most $m(n)$. Let $p(n)=9m(n)+n$. Let $S = \{(\bx_i,y_i)\}_{i=1}^{p(n)} \in (\reals^n \times \{0,1\})^{p(n)}$ be a sample that is contained in $A$. We will show an efficient algorithm $\ca$ that distinguishes whether $S$ is scattered or $\Dmat$-realizable. This implies that the theorem holds for $d$ such that $n^d \geq p(n)$.

Given $S$, the algorithm $\ca$ learns a function $h:\reals^n \rightarrow \reals$ by running $\cl$ with an examples oracle that generates examples by choosing a random (uniformly distributed) example $(\bx_i,y_i) \in S$. We denote $\ell_S(h)=\frac{1}{p(n)}\sum_{i \in [p(n)]}\left(h(\bx_i)-y_i\right)^2$.
Now, if $\ell_S(h) \leq \frac{1}{10}$, then $\ca$ returns that $S$ is $\Dmat$-realizable, and otherwise it returns that it is scattered.
Clearly, the algorithm $\ca$ runs in polynomial time. We now show that if $S$ is $\Dmat$-realizable then $\ca$ recognizes it with probability at least $\frac{3}{4}$, and that if $S$ is scattered then it also recognizes it with probability at least $\frac{3}{4}$.

Assume first that $S$ is $\Dmat$-realizable. Let $\cd_S$ be the uniform distribution over $\bx_i \in \reals^n$ from $S$. In this case, since $\cd_S$ is supported on $A$, we are guaranteed that with probability at least $\frac{3}{4}$ over the choice of $W$ and the internal randomness of $\cl$, we have $\ell_S(h)=\E_{\bx \sim \cd_S}\left[\left(h(\bx)-h_W(\bx)\right)^2\right] \leq \frac{1}{10}$. Therefore, the algorithm returns ``$\Dmat$-realizable".

Now, assume that $S$ is scattered. Let $h:\reals^n \rightarrow \reals$ be the function returned by $\cl$. Let $h':\reals^n \rightarrow \{0,1\}$ be the following function. For every $\bx \in \reals^n$, if $h(\bx) \geq \frac{1}{2}$ then $h'(\bx)=1$, and otherwise $h'(\bx)=0$. Note that for every $(\bx_i,y_i) \in S$, if $h'(\bx_i) \neq y_i$ then $\left(h(\bx_i)-y_i\right)^2 \geq \frac{1}{4}$. Therefore, $\ell_S(h) \geq \frac{1}{4}\ell_S(h')$. Let $C \subseteq [p(n)]$ be the set of indices of $S$ that were not observed by $\cl$. Note that given $C$, the events $\{h'(\bx_i) = y_i\}_{i \in C}$ are independent from one another, and each has probability $\frac{1}{2}$. By the Hoefding bound, we have that $h'(\bx_i) \neq y_i$ for at most $\frac{1}{2} - \sqrt{\frac{\ln(n)}{n}}$ fraction of the indices in $C$ with probability at most
\[
\exp\left(-\frac{2|C|\ln(n)}{n}\right)=\exp\left(-\frac{2(8m(n)+n)\ln(n)}{n}\right) \leq \exp\left(-2\ln(n)\right)=\frac{1}{n^2}~.
\]
Thus, $h'(\bx_i) \neq y_i$ for at least $\frac{1}{2} - o_n(1)$ fraction of the indices in $C$ with probability at least $1-o_n(1)$.
Hence,
\[
\ell_S(h) \geq \frac{1}{4} \ell_S(h') \geq
\frac{1}{4} \cdot \frac{|C|}{p(n)}\left(\frac{1}{2}-o_n(1)\right) =
\frac{1}{4} \cdot \frac{8m(n)+n}{9m(n)+n}\left(\frac{1}{2}-o_n(1)\right) \geq
\frac{1}{9}-o_n(1)~.
\]
Therefore, for large enough $n$, with probability at least $\frac{3}{4}$ we have $\ell_S(h) > \frac{1}{10}$, and thus the algorithm returns ``scattered".
\end{proof}

\subsection{$\scat_{n^d}^A(\Hcnn^{n,m})$ is RSAT-hard}

For a predicate
$P:\{\pm 1\}^K\to\{0,1\}$ we
denote by $\csp(P,\neg P)$ the problem whose instances are collections, $J$, of constraints, each of which is either $P$ or $\neg P$ constraint, and the goal is to maximize the number of satisfied constraints.
Denote by
$\csp^{\rand}_{m(n)}(P,\neg P)$ the problem of distinguishing\footnote{As in $\csp^{\rand}_{m(n)}(P)$, in order to succeed, and algorithm must return ``satisfiable" w.p. at least $\frac{3}{4}-o_n(1)$ on every satisfiable formula and ``random" w.p. at least $\frac{3}{4} - o_n(1)$ on random formulas.} satisfiable from random formulas with $n$ variables and $m(n)$ constraints. Here, in a random formula, each constraint is chosen w.p. $\frac{1}{2}$ to be a uniform $P$ constraint and w.p. $\frac{1}{2}$ a uniform $\neg P$ constraint.

We will consider the predicate $T_{K,M}:\{0,1\}^{KM}\to\{0,1\}$ defined by
\[
T_{K,M}(z)=
\left(z_1\vee\ldots\vee z_K\right)\wedge
\left(z_{K+1}\vee\ldots\vee z_{2K}\right)
\wedge
\ldots
\wedge
\left(z_{(M-1)K+1}\vee\ldots\vee z_{MK}\right)~.
\]

We will need the following lemma from \cite{daniely2016complexity}. For an overview of its proof, see Appendix~\ref{appendix DNF proof}.

\begin{lemma} \cite{daniely2016complexity}
\label{lemma:from daniely}
Let $q(n)=\omega(\log(n))$ with $q(n) \leq \frac{n}{\log(n)}$, and let $d$ and $K$ be fixed integers. The problem $\csp^{\rand}_{n^d}(\sat_K)$ can be efficiently reduced to the problem $\csp^{\rand}_{n^{d-1}}(T_{K,q(n)},\neg T_{K,q(n)})$.
\end{lemma}


In the following lemma, we use Lemma~\ref{lemma:from daniely} in order to show RSAT-hardness of $\scat_{n^d}^A(\Hcnn^{n,m})$ with some appropriate $m$ and $A$.

\begin{lemma}
\label{lemma:scat Hcnn}
Let $n=(n'+1)\log^2(n')$, and let $d$ be a fixed integer. The problem $\scat_{n^d}^A(\Hcnn^{n,\log^2(n')})$, where $A$ is the ball of radius $\log^2(n')$ in $\reals^n$, is RSAT-hard.
\stam{
Let $S = \{(\bx_i,y_i)\}_{i=1}^{(n')^d} \in (\reals^{n'} \times \{0,1\})^{(n')^d}$ be a sample, where $d$ is fixed, and $n'=(n+1)\log^2(n)$. Distinguishing whether $S$ is scattered or $\Hcnn^{n',\log^2(n)}$-realizable is RSAT-hard, already when $\norm{\bx_i} \leq \log^2(n)$ for every $i \in [(n')^d]$.
}
\end{lemma}
\begin{proof}
By Assumption~\ref{hyp:only_sat}, there is $K$ such that $\csp^{\rand}_{(n')^{d+2}}(\sat_K)$ is hard, where the $K$-SAT formula is over $n'$ variables.
Then, by Lemma~\ref{lemma:from daniely}, the problem $\csp^{\rand}_{(n')^{d+1}}(T_{K,\log^2(n')},\neg T_{K,\log^2(n')})$ is also hard.
We will reduce $\csp^{\rand}_{(n')^{d+1}}(T_{K,\log^2(n')},\neg T_{K,\log^2(n')})$ to
$\scat_{(n')^{d+1}}^A(\Hcnn^{n,\log^2(n')})$.
Since $(n')^{d+1}>n^d$, it would imply that $\scat_{n^d}^A(\Hcnn^{n,\log^2(n')})$ is RSAT-hard.

Let $J=\{C_1,\ldots,C_{(n')^{d+1}}\}$ be an input for $\csp^{\rand}_{(n')^{d+1}}(T_{K,\log^2(n')},\neg T_{K,\log^2(n')})$. Namely, each constraint $C_i$ is either a CNF or a DNF formula.
Equivalently, $J$ can be written as $J'=\{(C'_1,y_1),\ldots,(C'_{(n')^{d+1}},y_{(n')^{d+1}})\}$ where for every $i$, if $C_i$ is a DNF formula then $C'_i=C_i$ and $y_i=1$, and if $C_i$ is a CNF formula then $C'_i$ is the DNF obtained by negating $C_i$, and $y_i=0$. Given $J'$ as above, we encode each DNF formula $C'_i$ (with $\log^2(n')$ clauses) as a vector $\bx_i \in \reals^n$ such that each clause $[(\alpha_1,i_1),\ldots,(\alpha_K,i_K)]$ in $C'_i$ (a signed $K$-tuple) is encoded by a vector $\bz = (z_1,\ldots,z_{n'+1})$ as follows. First, we have $z_{n'+1}=-(K-1)$. Then, for every $1 \leq j \leq K$ we have $z_{i_j}=\alpha_j$, and for every variable $l$ that does not appear in the clause we have $z_l=0$. Thus, for every $1 \leq l \leq n'$, the value of $z_l$ indicates whether the $l$-th variable appears in the clause as a positive literal, a negative literal, or does not appear. The encoding $\bx_i$ of $C'_i$ is the concatenation of the encodings of its clauses.

Let $S=\{(\bx_1,y_1),\ldots,(\bx_{(n')^{d+1}},y_{(n')^{d+1}})\}$.
If $J$ is random then $S$ is scattered, since each constraint $C_i$ is with probability $\frac{1}{2}$ a DNF formula, and with probability $\frac{1}{2}$ a CNF formula, and this choice is independent of the choice of the literals in $C_i$.
Assume now that $J$ is satisfiable by an assignment $\psi \in \{\pm 1\}^{n'}$. Let $\bw=(\psi,1) \in \{\pm 1\}^{n'+1}$. Note that $S$ is realizable by the CNN $h_\bw^n$ with $\log^2(n')$ hidden neurons. Indeed, if $\bz \in \reals^{n'+1}$ is the encoding of a clause of $C'_i$, then $\inner{\bz,\bw}=1$ if all the $K$ literals in the clause are satisfied by $\psi$, and otherwise $\inner{\bz,\bw} \leq -1$. Therefore, $h_\bw^n(\bx_i)=y_i$.

Note that by our construction, for every $i \in [(n')^{d+1}]$ we have for large enough $n'$
\[
\norm{\bx_i} = \sqrt{\log^2(n')\left(K+(K-1)^2\right)} \leq \log(n') \cdot K \leq \log^2(n')~.
\]
\end{proof}

\subsection{Hardness of learning random fully-connected neural networks}

Let $n=(n'+1)\log^2(n')$.
We say that a matrix $M$ of size $n \times n$ is a {\em diagonal-blocks matrix} if
\[
M = \begin{bmatrix}
    B^{11} & \dots & B^{1\log^2(n')} \\
    \vdots & \ddots & \vdots \\
    B^{\log^2(n')1} & \dots & B^{\log^2(n')\log^2(n')}
    \end{bmatrix}
\]
where each block $B^{ij}$ is a diagonal matrix $\diag(z_1^{ij},\ldots,z_{n'+1}^{ij})$.
For every $1 \leq i \leq n'+1$ let $S_i=\{i + j(n'+1): 0 \leq j \leq \log^2(n')-1\}$. Let $M_{S_i}$ be the submatrix of $M$ obtained by selecting the rows and columns in $S_i$. Thus, $M_{S_i}$ is a matrix of size $\log^2(n') \times \log^2(n')$. For $\bx \in \reals^n$ let $\bx_{S_i} \in \reals^{\log^2(n')}$ be the restriction of $\bx$ to the coordinates $S_i$.

\begin{lemma}
\label{lemma:diagonal blocks singular}
Let $M$ be a diagonal-blocks matrix. Then,
\[
s_{\min}(M) \geq \min_{1 \leq i \leq n'+1}s_{\min}(M_{S_i})~.
\]
\end{lemma}
\begin{proof}
For every $\bx \in \reals^n$ with $\norm{\bx}=1$ we have
\begin{eqnarray*}
\norm{M\bx}^2
&=& \sum_{1 \leq i \leq n'+1}\norm{M_{S_i}\bx_{S_i}}^2
\geq \sum_{1 \leq i \leq n'+1}\left(s_{\min}(M_{S_i})\norm{\bx_{S_i}}\right)^2
\\
&\geq& \min_{1 \leq i \leq n'+1}\left(s_{\min}(M_{S_i})\right)^2\sum_{1 \leq i \leq n'+1}\norm{\bx_{S_i}}^2
= \left(\min_{1 \leq i \leq n'+1}\left(s_{\min}(M_{S_i})\right)^2\right) \norm{\bx}^2
\\
&=& \min_{1 \leq i \leq n'+1}\left(s_{\min}(M_{S_i})\right)^2~.
\end{eqnarray*}
Hence, $s_{\min}(M) \geq \min_{1 \leq i \leq n'+1}s_{\min}(M_{S_i})$.
\end{proof}

\stam{
\begin{lemma}
Let $M$ be a diagonal blocks matrix. For every line $i$ in $M$ we denote $i=(b-1)(n+1)+r$, where $b,r$ are integers and $1 \leq r \leq n+1$. Thus, $b$ is the line index of the blocks in $M$ that correspond to the $i$-th line in $M$, and $r$ is the line index within the blocks.
Let $h_\bw^{n'}$ be a CNN such that $\bw \in \{\pm 1\}^{n+1}$, and let $W \in \reals^{n' \times \log^2(n)}$ such that $h_\bw^{n'}=h_W$.
Let $W'=MW$.
\end{lemma}
\begin{proof}
Let $M^\top=(\bv^1,\ldots,\bv^{n'})$.  Now, note that
\begin{eqnarray*}
W'_{ij}
&=& \inner{\bv^i,\bw^j}
= \inner{\left(\bv^i_{(j-1)(n+1)+1},\ldots,\bv^i_{j(n+1)}\right),\bw}
= \inner{(B^{b j}_{r 1},\ldots,B^{b j}_{r (n+1)}),\bw}
\\
&=& B^{b j}_{r r} \cdot \bw_{r} = z^{b j}_{r} \cdot \bw_{r}~.
\end{eqnarray*}
\end{proof}
}

\subsubsection{Proof of Theorem~\ref{thm:nn iid}}

Let $M$ be a diagonal-blocks matrix, where each block $B^{ij}$ is a diagonal matrix $\diag(z_1^{ij},\ldots,z_{n'+1}^{ij})$.
Assume that for all $i,j,l$ the entries $z_l^{ij}$ are i.i.d. copies of a random variable $z$ that has a symmetric distribution $\cd_z$ with variance $\sigma^2$. Also, assume that the random variable $z'=\frac{z}{\sigma}$ is $b$-subgaussian for some fixed $b$.

\stam{
\begin{theorem} \cite{tao2010random}
\label{thm:tao}
Let $\xi$ be a real random variable with $\E\xi=0$ and $\E\xi^2=1$. Also, suppose that $\E|\xi|^{C_0} < \infty$ for some sufficiently large universal constant $C_0$. Let $M_n(\xi)$ denote the random $n \times n$ matrix whose entries are i.i.d. copies of $\xi$. Then, for all $t>0$ we have
\[
Pr\left(s_{\min}(M_n(\xi)) \leq \frac{t}{\sqrt{n}}\right) < t + \co(n^{-c})~,
\]
where $c>0$ is a universal constant.
\end{theorem}
}

\begin{lemma}
\label{lemma:M singular iid}
\[
Pr\left(s_{\min}(M) \leq \frac{\sigma}{n'\log^2(n')}\right) = o_n(1)~.
\]
\end{lemma}
\begin{proof}
Let $M'=\frac{1}{\sigma}M$.
By Lemma~\ref{lemma:diagonal blocks singular}, we have
\begin{equation}\label{eq:submatrix singular value}
s_{\min}(M') \geq \min_{1 \leq i \leq n'+1}s_{\min}(M'_{S_i})~.
\end{equation}
Note that for every $i$, all entries of the matrix $M'_{S_i}$ are i.i.d. copies of $z'$.

Now, we need the following theorem:
\begin{theorem} \cite{rudelson2008littlewood}
\label{thm:subgaussian}
Let $\xi$ be a real random variable with expectation $0$ and variance $1$, and assume that $\xi$ is $b$-subgaussian for some $b>0$. Let $A$ be an $n \times n$ matrix whose entries are i.i.d. copies of $\xi$. Then, for every $t \geq 0$ we have
\[
Pr\left(s_{\min}(A) \leq \frac{t}{\sqrt{n}}\right) \leq C t + c^n
\]
where $C>0$ and $c \in (0,1)$ depend only on $b$.
\end{theorem}

By Theorem~\ref{thm:subgaussian}, since each matrix $M'_{S_i}$ is of size $\log^2(n') \times \log^2(n')$, we have for every $i \in [n'+1]$ that
\[
Pr\left(s_{\min}(M'_{S_i}) \leq \frac{t}{\log(n')}\right) \leq C t + c^{\log^2(n')}~.
\]
By choosing $t = \frac{1}{n'\log(n')}$ we have
\[
Pr\left(s_{\min}(M'_{S_i}) \leq \frac{1}{n'\log^2(n')}\right) \leq \frac{C}{n'\log(n')} + c^{\log^2(n')}~.
\]
Then, by the union bound we have
\[
Pr\left(\min_{1 \leq i \leq n'+1}\left(s_{\min}(M'_{S_i})\right) \leq \frac{1}{n'\log^2(n')}\right) \leq \frac{C(n'+1)}{n'\log(n')} + c^{\log^2(n')}(n'+1)=o_n(1)~.
\]
Combining this with $s_{\min}(M) = \sigma \cdot s_{\min}(M')$ and with Eq.~\ref{eq:submatrix singular value}, we have
\begin{eqnarray*}
Pr\left(s_{\min}(M) \leq \frac{\sigma}{n'\log^2(n')}\right)
&=& Pr\left(s_{\min}(M') \leq \frac{1}{n'\log^2(n')}\right)
\\
&\leq& Pr\left(\min_{1 \leq i \leq n'+1}\left(s_{\min}(M'_{S_i})\right)\leq \frac{1}{n'\log^2(n')}\right)
= o_n(1)~.
\end{eqnarray*}
\end{proof}


\begin{lemma}
\label{lemma:nn iid scat}
Let $\Dmat$ be a distribution over $\reals^{n \times \log^2(n')}$ such that each entry is drawn i.i.d. from $\cd_z$. Note that a $\Dmat$-random network $h_W$ has $\log^2(n')=\co(\log^2(n))$ hidden neurons.
Let $d$ be a fixed integer. Then, $\scat_{n^d}^A(\Dmat)$ is RSAT-hard, where $A$ is the ball of radius $\frac{n\log^2(n)}{\sigma}$ in $\reals^{n}$.
\end{lemma}
\begin{proof}
By Lemma~\ref{lemma:scat Hcnn}, the problem $\scat_{n^d}^{A'}(\Hcnn^{n,\log^2(n')})$ where $A'$ is the ball of radius $\log^2(n')$ in $\reals^n$, is RSAT-hard.
We will reduce this problem to $\scat_{n^d}^A(\Dmat)$. Given a sample $S = \{(\bx_i,y_i)\}_{i=1}^{n^d} \in (\reals^n \times \{0,1\})^{n^d}$ with $\norm{\bx_i} \leq \log^2(n')$ for every $i \in [n^d]$, we will, with probability $1-o_n(1)$, construct a sample $S'$ that is contained in $A$, such that if $S$ is scattered then $S'$ is scattered, and if $S$ is $\Hcnn^{n,\log^2(n')}$-realizable then $S'$ is $\Dmat$-realizable. Note that our reduction is allowed to fail with probability $o_n(1)$. Indeed, distinguishing scattered from realizable requires success with probability $\frac{3}{4}-o_n(1)$ and therefore reductions between such problems are not sensitive to a failure with probability $o_n(1)$.


Assuming that $M$ is invertible (note that by Lemma~\ref{lemma:M singular iid} it holds with probability $1-o_n(1)$), let $S'=\{(\bx'_1,y_1),\ldots,(\bx'_{n^d},y_{n^d})\}$ where for every $i \in [n^d]$ we have $\bx'_i=(M^\top)^{-1} \bx_i$.
Note that if $S$ is scattered then $S'$ is also scattered.

Assume that $S$ is realizable by the CNN $h_\bw^n$ with $\bw \in \{\pm 1\}^{n'+1}$. Let $W$ be the matrix of size $n \times \log^2(n')$ such that $h_W=h_\bw^n$. Thus, $W=(\bw^1,\ldots,\bw^{\log^2(n')})$ where for every $i \in [\log^2(n')]$ we have $(\bw^i_{(i-1)(n'+1)+1},\ldots,\bw^i_{i(n'+1)})=\bw$, and $\bw^i_j=0$ for every other $j \in [n]$. Let $W' = MW$. Note that $S'$ is realizable by $h_{W'}$. Indeed, for every $i \in [n^d]$ we have $y_i=h_\bw^n(\bx_i)=h_W(\bx_i)$, and $W^\top \bx_i = W^\top M^\top (M^\top)^{-1} \bx_i = (W')^\top \bx'_i$. Also, note that the entries of $W'$ are i.i.d. copies of $z$. Indeed, denote $M^\top=(\bv^1,\ldots,\bv^n)$. Then, for every line $i \in [n]$ we denote $i=(b-1)(n'+1)+r$, where $b,r$ are integers and $1 \leq r \leq n'+1$. Thus, $b$ is the line index of the block in $M$ that correspond to the $i$-th line in $M$, and $r$ is the line index within the block. Now, note that
\begin{eqnarray*}
W'_{ij}
&=& \inner{\bv^i,\bw^j}
= \inner{\left(\bv^i_{(j-1)(n'+1)+1},\ldots,\bv^i_{j(n'+1)}\right),\bw}
= \inner{(B^{b j}_{r 1},\ldots,B^{b j}_{r (n'+1)}),\bw}
\\
&=& B^{b j}_{r r} \cdot \bw_{r} = z^{b j}_{r} \cdot \bw_{r}~.
\end{eqnarray*}
Since $\cd_z$ is symmetric and $\bw_{r} \in \{\pm 1\}$, we have $W'_{ij} \sim \cd_z$ independently from the other entries. Thus, $W' \sim \Dmat$. Therefore, $h_{W'}$ is a $\Dmat$-random network.

By Lemma~\ref{lemma:M singular iid}, we have with probability $1-o_n(1)$ that for every $i \in [n^d]$,
\begin{eqnarray*}
\norm{\bx'_i}
&=& \norm{(M^\top)^{-1} \bx_i}
\leq s_{\max}\left((M^\top)^{-1}\right) \norm{\bx_i}
= \frac{1}{s_{\min}(M^\top)} \norm{\bx_i}
= \frac{1}{s_{\min}(M)} \norm{\bx_i}
\\
&\leq& \frac{n'\log^2(n')}{\sigma} \log^2(n')
\leq \frac{n\log^2(n)}{\sigma}~.
\end{eqnarray*}
\end{proof}

Finally, Theorem~\ref{thm:nn iid} follows immediately from Theorem~\ref{thm:learn vs scat} and the following lemma.

\begin{lemma}
\label{lemma:nn iid scat2}
Let $\Dmat$ be a distribution over $\reals^{\tn \times m}$ with $m=\co(\log^2(\tn))$, such that each entry is drawn i.i.d. from $\cd_z$.
Let $d$ be a fixed integer, and let $\epsilon>0$ be a small constant. Then, $\scat_{\tn^d}^A(\Dmat)$ is RSAT-hard, where $A$ is the ball of radius $\frac{\tn^{\epsilon}}{\sigma}$ in $\reals^{\tn}$.
\end{lemma}
\begin{proof}
For integers $k,l$ we denote by $\Dmat^{k,l}$ the distribution over $\reals^{k \times l}$ such that each entry is drawn i.i.d. from $\cd_z$. Let $c=\frac{2}{\epsilon}$, and let $\tn=n^c$.
By Lemma~\ref{lemma:nn iid scat}, the problem $\scat_{n^{cd}}^{A'}(\Dmat^{n,m})$ is RSAT-hard, where $m=\co(\log^2(n))$, and $A'$ is the ball of radius $\frac{n\log^2(n)}{\sigma}$ in $\reals^{n}$. We reduce this problem to $\scat_{\tn^d}^A(\Dmat^{\tn,m})$, where $A$ is the ball of radius $\frac{\tn^{\epsilon}}{\sigma}$ in $\reals^{\tn}$. Note that $m=\co(\log^2(n))=\co(\log^2(\tn))$.

Let $S = \{(\bx_i,y_i)\}_{i=1}^{n^{cd}} \in (\reals^{n} \times \{0,1\})^{n^{cd}}$ with $\norm{\bx_i} \leq \frac{n\log^2(n)}{\sigma}$.
For every $i \in [n^{cd}]$, let $\bx'_i \in \reals^\tn$ be the vector obtained from $\bx_i$ by padding it with zeros. Thus, $\bx'_i=(\bx_i,0,\ldots,0)$. Note that $n^{cd} = \tn^d$. Let $S' = \{(\bx'_i,y_i)\}_{i=1}^{\tn^d}$.
If $S$ is scattered then $S'$ is also scattered. Note that if $S$ is realizable by $h_W$ then $S'$ is realizable by $h_{W'}$ where $W'$ is obtained from $W$ by appending $\tn-n$ arbitrary lines. Assume that $S$ is $\Dmat^{n,m}$-realizable, that is, $W \sim \Dmat^{n,m}$. Then, $S'$ is realizable by $h_{W'}$ where $W'$ is obtained from $W$ by appending lines such that each component is drawn i.i.d. from $\cd_z$, and therefore, $S'$ is $\Dmat^{\tn,m}$-realizable.
Finally, for every $i \in \tn^d$ we have
\[
\norm{\bx'_i} = \norm{\bx_i} \leq \frac{n\log^2(n)}{\sigma} = \frac{\tn^{\frac{1}{c}} \log^2(\tn^{\frac{1}{c}})}{\sigma} \leq \frac{\tn^{\frac{2}{c}}}{\sigma} = \frac{\tn^{\epsilon}}{\sigma}~.
\]
\end{proof}

\subsubsection{Proof of Theorem~\ref{thm:nn normal}}

Let $\Dmat$ be a distribution over $\reals^{n \times m}$ with $m=\log^2(n)$, such that each entry is drawn i.i.d. from $\cn(0,1)$.
Let $d$ be a fixed integer.
By Lemma~\ref{lemma:nn iid scat2}, we have that $\scat_{n^d}^A(\Dmat)$ is RSAT-hard, where $A$ is the ball of radius $n^{\epsilon}$ in $\reals^n$.
Let $(\cn(0,1))^n$ be the distribution over $\reals^n$ where each component is drawn i.i.d. from $\cn(0,1)$. Recall that $(\cn(0,1))^n=\cn(\zero,I_n)$ (\cite{tong2012multivariate}). Therefore, in the distribution $\Dmat$, the columns are drawn i.i.d. from $\cn(\zero,I_n)$.
Let $\Dmat'$ be a distribution over $\reals^{n \times m}$, such that each column is drawn i.i.d. from $\cn(\zero,\Sigma)$.
By Theorem~\ref{thm:learn vs scat}, we need to show that $\scat_{n^d}^{A'}(\Dmat')$ is RSAT-hard, where $A'$ is the ball of radius $\frac{n^{\epsilon}}{\sqrt{\lambda_{\min}}}$ in $\reals^n$.
We show a reduction from $\scat_{n^d}^A(\Dmat)$ to $\scat_{n^d}^{A'}(\Dmat')$.

Let $S = \{(\bx_i,y_i)\}_{i=1}^{n^d} \in (\reals^n \times \{0,1\})^{n^d}$ be a sample.
Let $\Sigma = U \Lambda U^\top$ be the spectral decomposition of $\Sigma$, and let $M=U \Lambda^{\frac{1}{2}}$.
Recall that if $\bw \sim \cn(\zero,I_n)$ then $M\bw \sim \cn(\zero,\Sigma)$ (\cite{tong2012multivariate}).
For every $i \in [n^d]$, let $\bx'_i = (M^\top)^{-1} \bx_i$, and let $S'=\{(\bx'_1,y_1),\ldots,(\bx'_{n^d},y_{n^d})\}$.
Note that if $S$ is scattered then $S'$ is also scattered.
If $S$ is realizable by a $\Dmat$-random network $h_W$, then let $W'=MW$. Note that $S'$ is realizable by $h_{W'}$. Indeed, for every $i \in [n^d]$ we have $(W')^\top \bx'_i = W^\top M^\top (M^\top)^{-1} \bx_i = W^\top \bx_i$.
Let $W=(\bw_1,\ldots,\bw_{m})$ and let $W'=(\bw'_1,\ldots,\bw'_{m})$. Since $W' = MW$ then $\bw'_j = M \bw_j$ for every $j \in [m]$. Now, since $W \sim \Dmat$, we have for every $j$ that $\bw_j \sim \cn(\zero,I_n)$ (i.i.d.). Therefore, $\bw'_j = M \bw_j \sim \cn(\zero,\Sigma)$, and thus $W' \sim \Dmat'$. Hence, $S'$ is $\Dmat'$-realizable.

We now bound the norms of the vectors $\bx'_i$ in $S'$.
Note that for every $i \in [n^d]$ we have
\[
\norm{\bx'_i}
= \norm{(M^\top)^{-1} \bx_i}
= \norm{U \Lambda^{-\frac{1}{2}} \bx_i}
= \norm{\Lambda^{-\frac{1}{2}} \bx_i}
\leq \lambda_{\min}^{-\frac{1}{2}} \norm{\bx_i}
\leq \lambda_{\min}^{-\frac{1}{2}} n^{\epsilon}~.
\]

\subsubsection{Proof of Theorem~\ref{thm:nn sphere}}

Let $n=(n'+1)\log^2(n')$, and let $M$ be a diagonal-blocks matrix, where each block $B^{ij}$ is a diagonal matrix $\diag(z_1^{ij},\ldots,z_{n'+1}^{ij})$.
We denote $\bz^{ij}=(z_1^{ij},\ldots,z_{n'+1}^{ij})$, and $\bz^j = (\bz^{1j},\ldots,\bz^{\log^2(n') j}) \in \reals^n$. Note that for every $j \in [\log^2(n')]$, the vector $\bz^j$ contains all the entries on the diagonals of blocks in the $j$-th column of blocks in $M$. Assume that the vectors $\bz^j$ are drawn i.i.d. according to the uniform distribution on $r \cdot \bbs^{n-1}$.

\begin{lemma}
\label{lemma:M singular sphere}
For some universal constant $c'>0$ we have
\[
Pr\left(s_{\min}(M) \leq \frac{c'r}{n' \sqrt{n'} \log^5(n')} \right) = o_n(1)~.
\]
\end{lemma}
\begin{proof}
Let $M'=\frac{\sqrt{n}}{r} M$. For every $j \in [\log^2(n')]$, let $\tilde{\bz}^j \in \reals^n$ be the vector that contains all the entries on the diagonals of blocks in the $j$-th column of blocks in $M'$. That is, $\tilde{\bz}^j=\frac{\sqrt{n}}{r} \bz^j$. Note that the vectors $\tilde{\bz}^j$ are i.i.d. copies from the uniform distribution on $\sqrt{n} \cdot \bbs^{n-1}$.
By Lemma~\ref{lemma:diagonal blocks singular}, we have
\begin{equation}\label{eq:submatrix singular value2}
s_{\min}(M') \geq \min_{1 \leq i \leq n'+1}s_{\min}(M'_{S_i})~.
\end{equation}
Note that for every $i$, all columns of the matrix $M'_{S_i}$ are projections of the vectors $\tilde{\bz}^j$ on the $S_i$ coordinated. That is, the $j$-th column in $M'_{S_i}$ is obtained by drawing $\tilde{\bz}^j$ from the uniform distribution on $\sqrt{n} \cdot \bbs^{n-1}$ and projecting on the coordinates $S_i$.

We say that a distribution is {\em isotropic} if it has mean zero and its covariance matrix is the identity.
The covariance matrix of the uniform distribution on $\bbs^{n-1}$ is $\frac{1}{n}I_n$. Therefore, the uniform distribution on $\sqrt{n} \cdot \bbs^{n-1}$ is isotropic.
We will need the following theorem.
\begin{theorem} \cite{adamczak2012condition}
\label{thm:isotropic}
Let $m \geq 1$ and let $A$ be an $m \times m$ matrix with independent columns drawn from an isotropic log-concave distribution. For every $\epsilon \in (0,1)$ we have
\[
Pr\left(s_{\min}(A) \leq \frac{c \epsilon}{\sqrt{m}}\right) \leq Cm \epsilon
\]
where c and C are positive universal constants.
\end{theorem}

We show that the distribution of the columns of $M'_{S_i}$ is isotropic and log-concave. First, since the uniform distribution on $\sqrt{n} \cdot \bbs^{n-1}$ is isotropic, then its projection on a subset of coordinates is also isotropic, and thus the distribution of the columns of $M'_{S_i}$ is isotropic. In order to show that it is log-concave, we analyze its density.
Let $\bx \in \reals^n$ be a random variable whose distribution is the projection of a uniform distribution on $\bbs^{n-1}$ on $k$ coordinates. It is known that the probability density of $\bx$ is (see \cite{fang2018symmetric})
\[
f_{\bx}(x_1,\ldots,x_k) = \frac{\Gamma(n/2)}{\Gamma((n-k)/2)\pi^{k/2}} \left(1- \sum_{1 \leq i \leq k}x_i^2\right)^{\frac{n-k}{2}-1}~,
\]
where $\sum_{1 \leq i \leq k}x_i^2 < 1$.
Recall that the columns of $M'_{S_i}$ are projections of the uniform distribution over $\sqrt{n} \cdot \bbs^{n-1}$, namely, the sphere of radius $\sqrt{n}$ and not the unit sphere. Thus, let $\bx' = \sqrt{n} \bx$. The probability density of $\bx'$ is
\begin{eqnarray*}
f_{\bx'}(x'_1,\ldots,x'_k)
&=& \frac{1}{(\sqrt{n})^{k}} f_{\bx}\left(\frac{x'_1}{\sqrt{n}},\ldots,\frac{x'_k}{\sqrt{n}}\right)
\\
&=& \frac{1}{n^{k/2}} \cdot \frac{\Gamma(n/2)}{\Gamma((n-k)/2)\pi^{k/2}} \left(1- \sum_{1 \leq i \leq k}\left(\frac{x'_i}{\sqrt{n}}\right)^2\right)^{\frac{n-k}{2}-1}~,
\end{eqnarray*}
where $\sum_{1 \leq i \leq k}(x'_i)^2 < n$. We denote
\[
g(n,k) = \frac{1}{n^{k/2}} \cdot \frac{\Gamma(n/2)}{\Gamma((n-k)/2)\pi^{k/2}}~.
\]
By replacing $k$ with $\log^2(n')$ we have
\begin{eqnarray*}
f_{\bx'}(x'_1,\ldots,x'_{\log^2(n')})
&=& g(n,\log^2(n')) \left(1- \frac{1}{n} \sum_{1 \leq i \leq \log^2(n')}(x'_i)^2\right)^{\frac{n-\log^2(n')}{2}-1}~.
\end{eqnarray*}
Hence, we have
\begin{eqnarray*}
\log f_{\bx'}(x'_1,\ldots,x'_{\log^2(n')})
&=&
\\
\log \left(g(n,\log^2(n'))\right)
&+&
\left(\frac{n-\log^2(n')}{2}-1\right) \cdot \log \left(1- \frac{1}{n} \sum_{1 \leq i \leq \log^2(n')}(x'_i)^2\right)~.
\end{eqnarray*}
Since $\frac{n-\log^2(n')}{2}-1 > 0$, we need to show that the function
\begin{equation}\label{eq:log expression}
\log \left(1 - \frac{1}{n} \sum_{1 \leq i \leq \log^2(n')}(x'_i)^2\right)
\end{equation}
(where $\sum_{1 \leq i \leq \log^2(n')}(x'_i)^2 < n$) is concave. This function can be written as $h(f(x_1,\ldots,x_{\log^2(n')}))$, where
\[
h(x) = \log \left(1 + x\right),
\]
\[
f(x'_1,\ldots,x'_{\log^2(n')}) = - \frac{1}{n} \sum_{1 \leq i \leq \log^2(n')}(x'_i)^2~.
\]
Recall that if $h$ is concave and non-decreasing, and $f$ is concave, then their composition is also concave. Clearly, $h$ and $f$ satisfy these conditions, and thus the function in Eq.~\ref{eq:log expression} is concave. Hence $f_{\bx'}$ is log-concave.

We now apply Theorem~\ref{thm:isotropic} on $M'_{S_i}$, and obtain that for every $\epsilon \in (0,1)$ we have
\[
Pr\left(s_{\min}(M'_{S_i}) \leq \frac{c \epsilon}{\log(n')}\right) \leq C \log^2(n') \epsilon~.
\]
By choosing $\epsilon = \frac{1}{n' \log^3(n')}$ we have
\[
Pr\left(s_{\min}(M'_{S_i}) \leq \frac{c}{n'\log^4(n')}\right) \leq \frac{C}{n' \log(n')}~.
\]
Now, by the union bound
\[
Pr\left(\min_{1 \leq i \leq n'+1}(s_{\min}(M'_{S_i})) \leq \frac{c}{n'\log^4(n')}\right) \leq \frac{C}{n' \log(n')} \cdot (n'+1) = o_n(1)~.
\]
Combining this with $s_{\min}(M) = \frac{r}{\sqrt{n}} s_{\min}(M')$ and with Eq.~\ref{eq:submatrix singular value2}, we have
\begin{eqnarray*}
Pr\left(s_{\min}(M) \leq \frac{cr}{\sqrt{n} \cdot n'\log^4(n')} \right)
&=& Pr\left(s_{\min}(M') \leq \frac{c}{n'\log^4(n')} \right)
\\
&\leq& Pr\left(\min_{1 \leq i \leq n'+1}(s_{\min}(M'_{S_i})) \leq \frac{c}{n'\log^4(n')}\right)
= o_n(1)~.
\end{eqnarray*}
Note that
\[
\frac{cr}{\sqrt{n} \cdot n'\log^4(n')}
= \frac{cr}{\sqrt{n'+1} \cdot n'\log^5(n')}
\geq \frac{cr}{2\sqrt{n'} \cdot n'\log^5(n')}
=  \frac{c'r}{\sqrt{n'} \cdot n'\log^5(n')}~,
\]
where $c'=\frac{c}{2}$. Thus,
\[
Pr\left(s_{\min}(M) \leq \frac{c'r}{\sqrt{n'} \cdot n'\log^5(n')} \right) \leq
Pr\left(s_{\min}(M) \leq \frac{cr}{\sqrt{n} \cdot n'\log^4(n')} \right) =
o_n(1)~.
\]
\end{proof}

Let $\Dmat$ be a distribution over $\reals^{n \times \log^2(n')}$ such that each column is drawn i.i.d. from the uniform distribution on $r \cdot \bbs^{n-1}$. Note that a $\Dmat$-random network $h_W$ has $\log^2(n')=\co(\log^2(n))$ hidden neurons.
Now, Theorem~\ref{thm:nn sphere} follows immediately from Theorem~\ref{thm:learn vs scat} and the following lemma.

\begin{lemma}
Let $d$ be a fixed integer. Then, $\scat_{n^d}^A(\Dmat)$ is RSAT-hard, where $A$ is a ball of radius $\co\left(\frac{n \sqrt{n} \log^4(n)}{r}\right)$ in $\reals^n$.
\end{lemma}
\begin{proof}
By Lemma~\ref{lemma:scat Hcnn}, the problem $\scat_{n^d}^{A'}(\Hcnn^{n,\log^2(n')})$ where $A'$ is the ball of radius $\log^2(n')$ in $\reals^n$, is RSAT-hard.
We will reduce this problem to $\scat_{n^d}^A(\Dmat)$. Given a sample $S = \{(\bx_i,y_i)\}_{i=1}^{n^d} \in (\reals^n \times \{0,1\})^{n^d}$ with $\norm{\bx_i} \leq \log^2(n')$ for every $i \in [n^d]$, we will, with probability $1-o_n(1)$, construct a sample $S'$ that is contained in $A$, such that if $S$ is scattered then $S'$ is scattered, and if $S$ is $\Hcnn^{n,\log^2(n')}$-realizable then $S'$ is $\Dmat$-realizable. Note that our reduction is allowed to fail with probability $o_n(1)$. Indeed, distinguishing scattered from realizable requires success with probability $\frac{3}{4}-o_n(1)$ and therefore reductions between such problems are not sensitive to a failure with probability $o_n(1)$.

Assuming that $M$ is invertible (by Lemma~\ref{lemma:M singular sphere} it holds with probability $1-o_n(1)$), let $S'=\{(\bx'_1,y_1),\ldots,(\bx'_{n^d},y_{n^d})\}$ where for every $i$ we have $\bx'_i=(M^\top)^{-1} \bx_i$.
Note that if $S$ is scattered then $S'$ is also scattered.

Assume that $S$ is realizable by the CNN $h_\bw^n$ with $\bw \in \{\pm 1\}^{n'+1}$. Let $W$ be the matrix of size $n \times \log^2(n')$ such that $h_W=h_\bw^n$. Thus, $W=(\bw^1,\ldots,\bw^{\log^2(n')})$ where for every $i \in [\log^2(n')]$ we have $(\bw^i_{(i-1)(n'+1)+1},\ldots,\bw^i_{i(n'+1)})=\bw$, and $\bw^i_j=0$ for every other $j \in [n]$. Let $W' = MW$. Note that $S'$ is realizable by $h_{W'}$. Indeed, for every $i \in [n^d]$ we have $y_i=h_\bw^n(\bx_i)=h_W(\bx_i)$, and $W^\top \bx_i = W^\top M^\top (M^\top)^{-1} \bx_i = (W')^\top \bx'_i$. Also, note that the columns of $W'$ are i.i.d. copies from the uniform distribution on $r \cdot \bbs^{n-1}$. Indeed, denote $M^\top=(\bv^1,\ldots,\bv^{n})$. Then, for every line index $i \in [n]$ we denote $i=(b-1)(n'+1)+r$, where $b,r$ are integers and $1 \leq r \leq n'+1$. Thus, $b$ is the line index of the block in $M$ that correspond to the $i$-th line in $M$, and $r$ is the line index within the block. Now, note that
\begin{eqnarray*}
W'_{ij}
&=& \inner{\bv^i,\bw^j}
= \inner{\left(\bv^i_{(j-1)(n'+1)+1},\ldots,\bv^i_{j(n'+1)}\right),\bw}
= \inner{(B^{b j}_{r 1},\ldots,B^{b j}_{r (n'+1)}),\bw}
\\
&=& B^{b j}_{r r} \cdot \bw_{r} = z^{b j}_{r} \cdot \bw_{r}~.
\end{eqnarray*}
Since $\bw_{r} \in \{\pm 1\}$, and since the uniform distribution on a sphere does not change by multiplying a subset of component by $-1$, then the $j$-th column of $W'$ has the same distribution as $\bz^j$, namely, the uniform distribution over $r \cdot \bbs^{n-1}$. Also, the columns of $W'$ are independent. Thus, $W' \sim \Dmat$, and therefore $h_{W'}$ is a $\Dmat$-random network.

By Lemma~\ref{lemma:M singular sphere}, we have with probability $1-o_n(1)$ that for every $i$,
\begin{eqnarray*}
\norm{\bx'_i}
&=& \norm{(M^\top)^{-1} \bx_i}
\leq s_{\max}\left((M^\top)^{-1}\right) \norm{\bx_i}
= \frac{1}{s_{\min}(M^\top)} \norm{\bx_i}
= \frac{1}{s_{\min}(M)} \norm{\bx_i}
\\
&\leq& \frac{n' \sqrt{n'} \log^5(n')}{c'r} \cdot \log^2(n')
\leq \frac{n \sqrt{n} \log^4(n)}{c'r}~.
\end{eqnarray*}
Thus, $\norm{\bx'_i} = \co\left(\frac{n \sqrt{n} \log^4(n)}{r}\right)$.
\end{proof}

\subsection{Hardness of learning random convolutional neural networks}

\stam{
\begin{lemma}
\label{lemma:common cnn trick}
Let $n'=(n+1)\log^2(n)$ and let $\bx \in \reals^{n'}$. Denote $\bx=(\bx_1,\ldots,\bx_{\log^2(n)})$ where for every $j$ we have $\bx_j \in \reals^{n+1}$. Let $M$ be an invertible matrix of size $(n+1) \times (n+1)$, and let $bx' \in \reals^{n'}$ be such that $\bx'=((M^\top)^{-1}\bx_1,\ldots,(M^\top)^{-1}\bx_{\log^2(n)})$.
\end{lemma}
}

\subsubsection{Proof of Theorem~\ref{thm:cnn iid}}

Theorem~\ref{thm:cnn iid} follows immediately from Theorem~\ref{thm:learn vs scat} and the following lemma:

\begin{lemma}
\label{lemma:cnn iid scat}
Let $d$ be a fixed integer. Then, $\scat_{n^d}^A(\cd_z^{n'+1},n)$ is RSAT-hard, where $A$ is the ball of radius $\frac{\log^2(n')}{f(n')}$ in $\reals^n$.
\end{lemma}
\begin{proof}
By Lemma~\ref{lemma:scat Hcnn}, the problem $\scat_{n^d}^{A'}(\Hcnn^{n,\log^2(n')})$ where $A'$ is the ball of radius $\log^2(n')$ in $\reals^n$, is RSAT-hard.
We will reduce this problem to $\scat_{n^d}^A(\cd_z^{n'+1},n)$. Given a sample $S = \{(\bx_i,y_i)\}_{i=1}^{n^d} \in (\reals^n \times \{0,1\})^{n^d}$ with $\norm{\bx_i} \leq \log^2(n')$ for every $i \in [n^d]$, we will, with probability $1-o_n(1)$, construct a sample $S'$ that is contained in $A$, such that if $S$ is scattered then $S'$ is scattered, and if $S$ is $\Hcnn^{n,\log^2(n')}$-realizable then $S'$ is $\cd_z^{n'+1}$-realizable. Note that our reduction is allowed to fail with probability $o_n(1)$. Indeed, distinguishing scattered from realizable requires success with probability $\frac{3}{4}-o_n(1)$ and therefore reductions between such problems are not sensitive to a failure with probability $o_n(1)$.

Let $\bz = (z_1,\ldots,z_{n'+1})$ where each $z_i$ is drawn i.i.d. from $\cd_z$. Let $M=\diag(\bz)$ be a diagonal matrix.
Note that $M$ is invertible with probability $1-o_n(1)$, since for every $i \in [n'+1]$ we have $Pr_{z_i \sim \cd_z}(z_i=0) \leq Pr_{z_i \sim \cd_z}(|z_i| < f(n'))=o(\frac{1}{n'})$. Now, for every $\bx_i$ from $S$, denote $\bx_i=(\bx^i_1,\ldots,\bx^i_{\log^2(n')})$ where for every $j$ we have $\bx^i_j \in \reals^{n'+1}$. Let $\bx'_i=(M^{-1}\bx^i_1,\ldots,M^{-1}\bx^i_{\log^2(n')})$, and let $S'=\{(\bx'_1,y_1),\ldots,(\bx'_{n^d},y_{n^d})\}$.
Note that if $S$ is scattered then $S'$ is also scattered.
If $S$ is realizable by a CNN $h_\bw^n \in \Hcnn^{n,\log^2(n')}$, then let $\bw'=M \bw$. Note that $S'$ is realizable by $h_{\bw'}^n$. Indeed, for every $i$ and $j$ we have
$\inner{\bw',M^{-1}\bx^i_j} = \bw^\top M^\top M^{-1}\bx^i_j = \bw^\top M M^{-1}\bx^i_j = \inner{\bw,\bx^i_j}$.
\stam{
\[
\inner{\bw',M^{-1}\bx^i_j} = \bw^\top M^\top M^{-1}\bx^i_j = \bw^\top M M^{-1}\bx^i_j = \inner{\bw,\bx^i_j}~.
\]
}
Also, note that since $\bw \in \{\pm 1\}^{n'+1}$ and $\cd_z$ is symmetric, then $\bw'$ has the distribution $\cd_z^{n'+1}$, and thus $h_{\bw'}^n$ is a $\cd_z^{n'+1}$-random CNN.

The probability that $\bz \sim \cd_z^{n'+1}$ has some component $z_i$ with $|z_i|<f(n')$, is at most $(n'+1) \cdot o(\frac{1}{n'}) = o_n(1)$. Therefore, with probability $1-o_n(1)$ we have for every $i \in [n^d]$ that
\begin{eqnarray*}
\norm{\bx'_i}^2
&=& \sum_{1 \leq j \leq \log^2(n')} \norm{M^{-1}\bx^i_j}^2
\leq \sum_{1 \leq j \leq \log^2(n')}\left(\frac{1}{f(n')}\norm{\bx^i_j}\right)^2
= \frac{1}{(f(n'))^2} \sum_{1 \leq j \leq \log^2(n')} \norm{\bx^i_j}^2
\\
&=& \frac{1}{(f(n'))^2} \norm{\bx_i}^2
\leq \frac{\log^4(n')}{(f(n'))^2}~.
\end{eqnarray*}
Thus, $\norm{\bx'_i} \leq \frac{\log^2(n')}{f(n')}$.
\end{proof}

\subsubsection{Proof of Theorem~\ref{thm:cnn normal}}


Assume that the covariance matrix $\Sigma$ is of size $(n'+1) \times (n'+1)$, and let $n=(n'+1)\log^2(n')$. Note that a $\cn(\zero,\Sigma)$-random CNN $h_\bw^{n}$ has $\log^2(n')=\co(\log^2(n))$ hidden neurons.
Let $\Dvec$ be a distribution over $\reals^{n'+1}$ such that each component is drawn i.i.d. from $\cn(0,1)$.  Let $d$ be a fixed integer.
By Lemma~\ref{lemma:cnn iid scat} and by choosing $f(n')=\frac{1}{n'\log(n')}$, we have that $\scat_{n^d}^A(\Dvec,n)$ is RSAT-hard, where $A$ is the ball of radius $n' \log^3(n') \leq n \log(n)$ in $\reals^{n}$.
Note that $\Dvec=\cn(\zero,I_{n'+1})$ (\cite{tong2012multivariate}).
By Theorem~\ref{thm:learn vs scat}, we need to show that $\scat_{n^d}^{A'}(\cn(\zero,\Sigma),n)$ is RSAT-hard, where $A'$ is the ball of radius $\lambda_{\min}^{-\frac{1}{2}} n \log(n)$ in $\reals^{n}$.
We show a reduction from $\scat_{n^d}^A(\cn(\zero,I_{n'+1}),n)$ to $\scat_{n^d}^{A'}(\cn(\zero,\Sigma),n)$.

Let $S = \{(\bx_i,y_i)\}_{i=1}^{n^d} \in (\reals^n \times \{0,1\})^{n^d}$ be a sample. For every $\bx_i$ from $S$, denote $\bx_i=(\bx^i_1,\ldots,\bx^i_{\log^2(n')})$ where for every $j$ we have $\bx^i_j \in \reals^{n'+1}$.
Let $\Sigma = U \Lambda U^\top$ be the spectral decomposition of $\Sigma$, and let $M=U \Lambda^{\frac{1}{2}}$.
Recall that if $\bw \sim \cn(\zero,I_{n'+1})$ then $M\bw \sim \cn(\zero,\Sigma)$ (\cite{tong2012multivariate}).
Let $\bx'_i=((M^\top)^{-1}\bx^i_1,\ldots,(M^\top)^{-1}\bx^i_{\log^2(n')})$, and let $S'=\{(\bx'_1,y_1),\ldots,(\bx'_{n^d},y_{n^d})\}$.
Note that if $S$ is scattered then $S'$ is also scattered.
If $S$ is realizable by a $\cn(\zero,I_{n'+1})$-random CNN $h_\bw^n$, then let $\bw'=M \bw$. Note that $S'$ is realizable by $h_{\bw'}^n$. Indeed, for every $i$ and $j$ we have
$\inner{\bw',(M^\top)^{-1}\bx^i_j} = \bw^\top M^\top (M^\top)^{-1} \bx^i_j = \inner{\bw,\bx^i_j}$.
\stam{
\[
\inner{\bw',(M^\top)^{-1}\bx^i_j} = \bw^\top M^\top (M^\top)^{-1} \bx^i_j = \inner{\bw,\bx^i_j}~.
\]
}
Since $\bw' = M \bw \sim \cn(\zero,\Sigma)$, the sample $S'$ is $\cn(\zero,\Sigma)$-realizable.

We now bound the norms of $\bx'_i$ in $S'$.
Note that for every $i \in [n^d]$ we have
\begin{eqnarray*}
\norm{\bx'_i}^2
&=& \sum_{1 \leq j \leq \log^2(n')}\norm{(M^\top)^{-1} \bx^i_j}^2
= \sum_{1 \leq j \leq \log^2(n')}\norm{U \Lambda^{-\frac{1}{2}} \bx^i_j}^2
= \sum_{1 \leq j \leq \log^2(n')}\norm{\Lambda^{-\frac{1}{2}} \bx^i_j}^2
\\
&\leq&  \sum_{1 \leq j \leq \log^2(n')}\norm{\lambda_{\min}^{-\frac{1}{2}} \bx^i_j}^2
= \lambda_{\min}^{-1} \sum_{1 \leq j \leq \log^2(n')}\norm{\bx^i_j}^2
= \lambda_{\min}^{-1} \norm{\bx_i}^2~.
\end{eqnarray*}
Hence, $\norm{\bx'_i} \leq \lambda_{\min}^{-\frac{1}{2}} \norm{\bx_i} \leq \lambda_{\min}^{-\frac{1}{2}} n \log(n)$.

\subsubsection{Proof of Theorem~\ref{thm:cnn sphere}}

Let $n=(n'+1)\log^2(n')$. Let $\Dvec$ be the uniform distribution on $r \cdot \bbs^{n'}$. Note that a $\Dvec$-random CNN $h_\bw^n$ has $\log^2(n')=\co(\log^2(n))$ hidden neurons. Let $d$ be a fixed integer.
By Theorem~\ref{thm:learn vs scat}, we need to show that $\scat_{n^d}^A(\Dvec,n)$ is RSAT-hard, where $A$ is the ball of radius $\frac{\sqrt{n}\log(n)}{r}$ in $\reals^{n}$.
By Lemma~\ref{lemma:scat Hcnn}, the problem $\scat_{n^d}^{A'}(\Hcnn^{n,\log^2(n')})$ where $A'$ is the ball of radius $\log^2(n')$ in $\reals^n$, is RSAT-hard.
We reduce this problem to $\scat_{n^d}^A(\Dvec,n)$.
Given a sample $S = \{(\bx_i,y_i)\}_{i=1}^{n^d} \in (\reals^n \times \{0,1\})^{n^d}$ with $\norm{\bx_i} \leq \log^2(n')$ for every $i \in [n^d]$, we construct a sample $S'$ that is contained in $A$, such that if $S$ is scattered then $S'$ is scattered, and if $S$ is $\Hcnn^{n,\log^2(n')}$-realizable then $S'$ is $\Dvec$-realizable.

Let $M$ be a random orthogonal matrix of size $(n'+1) \times (n'+1)$.
For every $i \in [n^d]$ denote $\bx_i=(\bx^i_1,\ldots,\bx^i_{\log^2(n')})$ where for every $j$ we have $\bx^i_j \in \reals^{n'+1}$. For every $i \in [n^d]$ let $\bx'_i=(\frac{\sqrt{n'+1}}{r}M \bx^i_1,\ldots,\frac{\sqrt{n'+1}}{r}M \bx^i_{\log^2(n')})$, and let $S'=\{(\bx'_1,y_1),\ldots,(\bx'_{n^d},y_{n^d})\}$.
Note that if $S$ is scattered then $S'$ is also scattered.
If $S$ is realizable by a CNN $h_\bw^n \in \Hcnn^{n,\log^2(n')}$, then let $\bw'=\frac{r}{\sqrt{n'+1}}M \bw$. Note that $S'$ is realizable by $h_{\bw'}^n$. Indeed, for every $i$ and $j$ we have
\[
\inner{\bw',\frac{\sqrt{n'+1}}{r}M\bx^i_j} = \bw^\top \frac{r}{\sqrt{n'+1}}M^\top \frac{\sqrt{n'+1}}{r} M \bx^i_j = \inner{\bw,\bx^i_j}~.
\]
Also, note that since $\norm{\bw}=\sqrt{n'+1}$ and $M$ is orthogonal, $\bw'$ is a random vector on the sphere of radius $r$ in $\reals^{n'+1}$, and thus $h_{\bw'}^n$ is a $\Dvec$-random CNN.

Since $M$ is orthogonal then for every $i \in [n^d]$ we have
\begin{eqnarray*}
\norm{\bx'_i}^2
&=& \sum_{1 \leq j \leq \log^2(n')} \norm{\frac{\sqrt{n'+1}}{r}M\bx^i_j}^2
= \frac{n'+1}{r^2} \sum_{1 \leq j \leq \log^2(n')} \norm{\bx^i_j}^2
\\
&=& \frac{n'+1}{r^2} \cdot \norm{\bx_i}^2
\leq \frac{(n'+1)\log^4(n')}{r^2}
\leq \frac{n\log^2(n)}{r^2}~.
\end{eqnarray*}
Hence $\norm{\bx'_i} \leq \frac{\sqrt{n}\log(n)}{r}$.

\bibliographystyle{abbrvnat}
\bibliography{bib}

\appendix

\section{From $\csp^{\rand}_{n^d}(\sat_K)$ to $\csp^{\rand}_{n^{d-1}}(T_{K,q(n)},\neg T_{K,q(n)})$ (\cite{daniely2016complexity})}
\label{appendix DNF proof}

We outline the main ideas of the reduction.

First, we reduce $\csp^{\rand}_{n^d}(\sat_K)$ to $\csp^{\rand}_{n^{d-1}}(T_{K,q(n)})$. This is done as follows. Given an instance $J=\{C_1,\ldots,C_{n^d}\}$ to $\csp(\sat_K)$, by a simple greedy procedure, we try to find $n^{d-1}$ disjoint subsets $J'_1,\ldots,J'_{n^{d-1}}\subset J$, such that for every $t$, the subset $J'_t$ consists of $q(n)$ constraints and each variable appears in at most one of the constraints in $J'_t$. Now, from every $J'_t$ we construct a $T_{K,q(n)}$-constraint that is the conjunction of all constraints in $J'_t$. If $J$ is random, this procedure will succeed w.h.p. and will produce a random $T_{K,q(n)}$-formula. If $J$ is satisfiable, this procedure will either fail or produce a satisfiable
$T_{K,q(n)}$-formula.

Now, we reduce $\csp^{\rand}_{n^{d-1}}(T_{K,q(n)})$ to $\csp^{\rand}_{n^{d-1}}(T_{K,q(n)},\neg T_{K,q(n)})$. This is done by replacing each constraint, with probability $\frac{1}{2}$, with a random $\neg P$ constraint. Clearly, if the original instance is a random instance of $\csp^{\rand}_{n^{d-1}}(T_{K,q(n)})$, then the produced instance is a random instance of $\csp^{\rand}_{n^{d-1}}(T_{K,q(n)},\neg T_{K,q(n)})$. Furthermore, if the original instance is satisfied by the assignment $\psi\in\{\pm 1\}^n$, the same $\psi$, w.h.p., will satisfy all the new constraints. The reason is that
the probability that a random $\neg T_{K,q(n)}$-constraint is satisfied by $\psi$ is $1-\left(1-2^{-K}\right)^{q(n)}$, and hence, the probability that all new constraints are satisfied by $\psi$ is at least $1-n^{d-1}\left(1-2^{-K}\right)^{q(n)}$. Now, since $q(n)=\omega(\log(n))$, the last probability is $1-o_n(1)$.

For the full proof see \cite{daniely2016complexity}.

\end{document}